\documentclass[11pt]{article}

\usepackage{fullpage,times,url,bm}

\usepackage{amsthm,amsfonts,amsmath,amssymb,epsfig,color,float,graphicx,verbatim}
\usepackage{algorithm,algorithmic}
\usepackage{bbm}
\usepackage{natbib}

\usepackage{graphicx}
\usepackage{subcaption}

\usepackage{hyperref}
\hypersetup{
	colorlinks   = true, %Colours links instead of ugly boxes
	urlcolor     = blue, %Colour for external hyperlinks
	linkcolor    = blue, %Colour of internal links
	citecolor   = black %Colour of citations
}

\newtheorem{theorem}{Theorem}[section]
\newtheorem{proposition}{Proposition}[section]
\newtheorem{lemma}{Lemma}[section]
\newtheorem{corollary}{Corollary}[section]
\newtheorem{definition}{Definition}[section]
\newtheorem{remark}{Remark}[section]

\usepackage{amssymb}

\usepackage{bbm}
\newcommand{\onefunc}{\mathbbm{1}}

\newcommand{\stam}[1]{}

\usepackage{mathtools}
\DeclarePairedDelimiter\ceil{\lceil}{\rceil}

\newcommand{\ba}{\mathbf{a}}

\newcommand{\bx}{\mathbf{x}}
\newcommand{\bw}{\mathbf{w}}

\newcommand{\bb}{\mathbf{b}}

\newcommand{\bz}{\mathbf{z}}
\newcommand{\bc}{\mathbf{c}}

\newcommand{\by}{\mathbf{y}}

\newcommand{\co}{{\cal O}}
\newcommand{\ca}{{\cal A}}

\newcommand{\cd}{{\cal D}}

\newcommand{\cc}{{\cal C}}

\newcommand{\cg}{{\cal G}}
\newcommand{\ch}{{\cal H}}

\newcommand{\ci}{{\cal I}}

\newcommand{\cf}{{\cal F}}

\newcommand{\cu}{{\cal U}}

\newcommand{\cn}{{\cal N}}

\DeclareMathOperator*{\sign}{sign}

\newcommand{\reals}{{\mathbb R}}

\newcommand{\integers}{{\mathbb Z}}

\DeclareMathOperator{\poly}{poly}

\DeclareMathOperator*{\E}{\mathbb{E}}

\DeclareMathOperator{\bin}{bin}
\DeclareMathOperator{\real}{real}

\DeclareMathOperator{\trunc}{trunc}

\newcommand{\inner}[1]{\langle #1 \rangle}
\newcommand{\norm}[1]{\left\|#1\right\|}
\newcommand{\snorm}[1]{\|#1\|} %small norm
\newcommand{\tbx}{{\tilde{\bx}}}
\newcommand{\tx}{{\tilde{x}}}
\newcommand{\hbx}{{\hat{\bx}}}
\newcommand{\hx}{{\hat{x}}}

\newcommand{\TC}{\textsf{TC}}

\newcommand{\Ppoly}{\textsf{P/poly}}
\newcommand{\NP}{\textsf{NP}}
\newcommand{\Ptime}{\textsf{P}}
\newcommand{\EXP}{\textsf{EXP}}
\newcommand{\NEXP}{\textsf{NEXP}}
\newcommand{\PSPACE}{\textsf{PSPACE}}

\title{Size and Depth Separation in Approximating Benign Functions\\
with Neural Networks}

\author{
	Gal Vardi\thanks{Weizmann Institute of Science, Israel, \texttt{gal.vardi@weizmann.ac.il}}
	\and
	Daniel Reichman\thanks{Worcester Polytechnic Institute, \texttt{daniel.reichman@gmail.com}}
	\and
	Toniann Pitassi\thanks{University of Toronto and IAS, \texttt{toni@cs.toronto.edu}}
	\and
	Ohad Shamir\thanks{Weizmann Institute of Science, Israel, \texttt{ohad.shamir@weizmann.ac.il}}
}

\date{}

\begin{document}

\maketitle

\begin{abstract}
When studying the expressive power of neural networks, a main challenge  is to understand how the size and depth of the network affect its ability to approximate real functions. 
However, not all functions are interesting from a practical viewpoint: functions of interest usually have a polynomially-bounded Lipschitz constant, and can be computed efficiently. We call functions that satisfy these conditions ``benign", and explore the benefits of size and depth for approximation of benign functions with ReLU networks.
As we show, this problem is more challenging than the corresponding problem for non-benign functions.
We give complexity-theoretic barriers to showing depth-lower-bounds: 
Proving existence of a benign function that cannot be approximated by polynomial-sized networks of depth $4$ would settle longstanding open problems in computational complexity. It implies that beyond depth $4$ there is a barrier to showing depth-separation for benign functions, even between networks of constant depth and networks of nonconstant depth. 
We also study size-separation, namely, whether there are benign functions that can be approximated with networks of size $\co(s(d))$, but not with networks of size $\co(s'(d))$. We show a complexity-theoretic barrier to proving such results beyond size $\co(d\log^2(d))$, but also show an explicit benign function, that can be approximated with networks of size $\co(d)$ and not with networks of size $o(d/\log d)$. For approximation in the $L_\infty$ sense we achieve such separation already between size $\co(d)$ and size $o(d)$.
Moreover, we show superpolynomial size lower bounds and barriers to such lower bounds, depending on the assumptions on the function.
Our size-separation results rely on an analysis of size lower bounds for Boolean functions, which is of independent interest: We show linear size lower bounds for computing explicit Boolean functions (such as set disjointness) with neural networks and threshold circuits. 
\end{abstract}

\section{Introduction}

The {\em expressive power} of feedforward neural networks is a central topic in the theory of deep learning. It is well-known that sufficiently large depth-$2$ neural networks, using reasonable activation functions, can approximate any continuous function on a bounded domain (\cite{cybenko1989approximation,funahashi1989approximate,hornik1991approximation,barron1994approximation}). However, the required size of such networks (namely, the overall number of neurons) can be impractically large, e.g., exponential in the input dimension. From a learning perspective, both theoretically and in practice, the main interest is in neural networks 
whose size is at most polynomial in the input dimension.

Many works in recent years have studied 
the expressive power of polynomial-size neural networks, and the beneficial effect of depth for approximating real functions.
However, not all functions are interesting from a practical viewpoint: For example, in practice we are interested in functions that can be efficiently computed. Moreover, in learning tasks, it is usually sufficient to consider prediction functions which have some polynomially-bounded Lipschitz parameter: Otherwise, it means that the learning task crucially requires a function that varies at a superpolynomial rate, which is generally not the case. In addition, functions with very large Lipschitz constants tend to be more difficult to learn with standard methods (cf.~\cite{safran2019depth,malach2021connection}). 

Motivated by this, the main goal of our paper is to explore the benefits of size and depth for approximation of \emph{benign} functions, which do satisfy the conditions above. Specifically, we say that a function $f:[0,1]^d \rightarrow [0,1]$ is benign if it satisfies the following conditions (stated slightly informally): (1) it is $\poly(d)$-Lipschitz; (2) there is an algorithm that for $\bx \in [0,1]^d$ given in binary representation, computes $f(\bx)$ in at most \emph{exponential} time, within 
$1/\text{poly}(d)$ precision.
Clearly, this computability requirement is very mild. A stronger (and still mild) assumption is to replace the exponential-time requirement by a polynomial-time requirement, in which case we will call such functions {\em polynomial-time benign}.
We provide several results, both positive and negative, on our ability to capture the benefits of size and depth for approximating benign functions with ReLU networks. 

\paragraph{Depth separation.}
Overwhelming empirical evidence indicates that deeper networks tend to perform better than shallow ones.
Quite a few theoretical works in recent years have explored the beneficial effect of depth on increasing the expressiveness of neural networks
(e.g., \cite{martens2013representational,eldan2016power,telgarsky2016benefits,liang2016deep,daniely2017depth,safran2017depth,yarotsky2017error,safran2019depth,vardi2020neural,bresler2020sharp}).
A main focus is on {\em depth separation}, namely, showing that there is a function $f:\reals^d \rightarrow \reals$ that can be approximated by a $\poly(d)$-sized network of a given depth, with respect to some input distribution, but cannot be approximated by $\poly(d)$-sized networks of a smaller depth.
Depth separation between depth $2$ and $3$ is known \citep{eldan2016power,daniely2017depth}\footnote{The result of \cite{daniely2017depth} holds only for depth-$2$ networks whose weights magnitudes are upper-bounded by an exponential.} already for benign functions.
A complexity-theoretic barrier to proving separation between two constant depths beyond depth $4$ was established in \cite{vardi2020neural}. 
A construction shown by \cite{telgarsky2016benefits} gives separation between $\poly(d)$-sized networks of a constant depth, and $\poly(d)$-sized networks of some nonconstant depth. Thus, restricting the depth hurts the expressiveness of $\poly(d)$-sized networks. However, the function of \cite{telgarsky2016benefits} is highly oscillatory, and its Lipschitz constant is %very large, namely, 
superpolynomial in $d$.
Hence, an interesting question is whether 
%such a result can be obtained for benign functions. That is, are 
$\poly(d)$-sized networks 
are
more powerful than $\poly(d)$-sized networks of constant depth, in their ability to approximate \emph{benign} functions.
We show:
\begin{itemize}
	\item  Proving existence of a benign function that cannot be approximated by $\poly(d)$-sized networks of constant depth $k \geq 4$, would settle a longstanding open problem in computational complexity (namely, $\EXP \not \subseteq \TC^0_{k-2}$, where $\TC^0_{k-2}$ denotes the class of threshold circuits of polynomial size and depth $k-2$).  Moreover, if we wish to prove the result for depth $k \geq 6$, it would require overcoming a \emph{natural-proofs barrier}, a concept from computational complexity which indicates that such a proof would be very hard to find.
	We note that \cite{vardi2020neural} gave a barrier to proving depth separation, and we give a barrier already to proving depth lower bounds. Thus, while the barrier of \cite{vardi2020neural} holds only for separation between two constant depths, our barrier applies also to separation between constant and nonconstant depths. 
	
	\item Interestingly, we show that this barrier crucially relies on both ``benignness'' requirements: If we allow exponential Lipschitz constants, then a depth lower bound is known \citep{telgarsky2016benefits}. Moreover, using counting arguments, we prove a lower bound for a function that is $1$-Lipschitz, but not necessarily computable in exponential time. 
\end{itemize}

\paragraph{Size separation.}
We study {\em size separation}, namely, whether there are benign functions that can be approximated by networks of size $\co(s(d))$, but cannot be approximated by networks of size $\co(s'(d))$. Here, we consider the overall number of neurons in the network, regardless of its depth/width.
Recall that the motivation behind the study of depth separation is to achieve a theoretical understanding of the empirical success of deeper networks compared to shallow ones. This motivation applies also to size separation, since often large networks are required in practice in order achieve good performance, and we are missing a theoretical understanding of this phenomenon. We show both size-separation results, and barriers to size-separation:
\begin{itemize}
	\item Proving existence of a polynomial-time benign function that cannot be approximated by networks of size $\co(d \log^2(d))$, would settle the longstanding open problem in circuit complexity, on whether there is a Boolean function in $\Ptime$ that cannot be computed by threshold circuits of linear size.
	
	\item We show a polynomial-time benign function, that can be approximated by a network of size $\co(d)$, but cannot be approximated by networks of size $o(d/\log d)$. 
	
	\item We also consider size-separation in the $L_\infty$ sense (where we wish to approximate the function uniformly rather than on average), and show a polynomial-time benign function that can be computed by a network of size $\co(d)$, but cannot be approximated by networks of size $o(d)$.  
\end{itemize}

\paragraph{Superpolynomial size lower bounds.}
Many works in recent years have studied approximation of classes of smooth functions with ReLU neural networks (e.g., \cite{guhring2020expressivity,yarotsky2017error,petersen2018optimal,yarotsky2018optimal,yarotsky2019phase,shen2019deep,lu2020deep}). The upper bounds on the required size of the network are at least exponential in the input dimension. \cite{yarotsky2017error} gave a lower bound for the required size for approximation in the $L_\infty$ sense, which is exponential in the input dimension. We show:
\begin{itemize}
	\item There is a $1$-Lipschitz function $f:[0,1]^d \rightarrow [0,1]$ that cannot be approximated by $\poly(d)$-sized networks whose weights are represented by a $\poly(d)$ number of bits. Thus, we obtain a superpolynomial size lower bound for approximating $1$-Lipschitz functions in the $L_2$ sense. However, this function is obtained by a counting argument and is not known to be benign.
	
	\item We give a barrier to proving superpolynomial size lower bounds, already for {\em semi-benign} functions, namely, functions with an \emph{exponential} Lipschitz constant, that are computable in exponential time.
	We show that proving such a lower bound would imply that $\EXP \not \subseteq \Ppoly$.
\end{itemize}

\paragraph{Size lower bounds for Boolean functions.}
Our size-separation results for benign functions rely on an analysis of the corresponding problem for Boolean functions. Namely, we show size lower-bounds and upper-bounds for computing certain Boolean functions with ReLU networks, and then use these bounds to obtain size-separation for real functions. 
We consider the Boolean functions that compute {\em disjointness} and {\em inner product}, and show linear size lower bounds. Our lower bounds are based on results from communication complexity, and hold also for networks with $k$-piecewise-linear activation.
We note that %our lower bounds for computing Boolean functions with neural networks 
these results 
are also of independent interest.
Indeed, 
the study of the computational power of 
neural networks in the context of Boolean functions has received ample attention
in the past decades (e.g.,
\cite{maass1991computational,koiran1996vc,maass1997bounds,martens2013representational,kane2016super,mukherjee2017lower,williams2018limits}).
%\cite{parberry1988parallel,hajnal1993threshold,parberry1994circuit,martens2013representational,kane2016super,mukherjee2017lower,williams2018limits}).
Moreover, our linear size lower bounds for computing disjointness and inner-product hold also for threshold circuits. This lower bound for threshold circuits was already shown with different methods for the inner-product function \citep{groeger1993linear,jukna2012boolean,roychowdhury1994lower}, but is new for disjointness.

\paragraph{Connection to threshold circuits.}
In order to establish our results, we explore the connection between ReLU networks and threshold circuits. These are essentially neural networks with a threshold activation function in all neurons (including the output neuron), and where the inputs are in $\{0,1\}^d$. Size and depth lower bounds for threshold circuits were extensively studied in the context of circuit complexity over the past decades.
Since a threshold circuit is a special case of a neural network, it is natural to ask whether the results on size and depth lower bounds in threshold circuits have implications on the analogous questions for neural networks, and indeed we study such implications in our work. However, we emphasize that in general, it is not obvious how to ``import'' separation results (or barriers) of threshold circuits to the realm of neural networks. This is because unlike threshold circuits, neural networks have real-valued inputs and outputs (not necessarily just Boolean ones), and a continuous activation function. Thus, it might be possible to come up with a separation result, which crucially utilizes some function and inputs in Euclidean space. In fact, this can already be seen in existing results: For example, separation between threshold circuits of polynomial size and constant depth ($\TC^0$) and threshold circuits of polynomial size of any depth (which equals the complexity class $\Ppoly$) is not known, but \cite{telgarsky2016benefits} showed such a result for neural networks. His construction is based on the observation that for one dimensional data, a network of depth $k$ is able to express a sawtooth function on the interval $[0,1]$ which oscillates $\co(2^k)$ times. Clearly, this utilizes the continuous structure of the domain, in a way that is not possible with Boolean inputs. Also, depth-separation results for neural networks \citep{eldan2016power,daniely2017depth} rely on harmonic analysis of real functions. Finally, the result of \cite{eldan2016power} does not make any assumption on the weight magnitudes, whereas relaxing this assumption for the parallel result on threshold circuits is a longstanding open problem \citep{razborov1992small}.

%\subsection*{Related work}

\section{Preliminaries}
\label{sec:preliminaries}

\textbf{Notations.}
We use bold-faced letters to denote vectors, e.g., $\bx=(x_1,\ldots,x_d)$. For $\bx \in \reals^d$ we denote by $\norm{\bx}$ the Euclidean norm. For a 
real
function 
$f$
%$f:\reals^d \rightarrow \reals$ 
and a distribution $\cd$,
% on $\reals^d$, 
%either continuous or discrete, 
we denote by $\snorm{f}_{L_2(\cd)}$ the $L_2$ norm weighted by $\cd$, namely $\snorm{f}_{L_2(\cd)}^2 = \E_{\bx \sim \cd}(f(\bx))^2$. 
%Given two functions $f,g$ and real numbers $\alpha,\beta$, we let $\alpha f + \beta g$ be shorthand for $\bx \mapsto \alpha f(\bx) + \beta g(\bx)$. 
For a set $A$, 
%we let $\onefunc_A$ denote the indicator function, i.e., $\onefunc_A(z)=1$ if $z \in A$ and $\onefunc_A(z)=0$ otherwise.
%We 
we
denote by $\cu(A)$ the uniform distribution over $A$.
For an integer $d \geq 1$ we denote $[d]=\{1,\ldots,d\}$. 
We use $\poly(d)$ as a shorthand for ``some polynomial in $d$".
Let $\mu$ be the density function of a continuous distribution on $[0,1]^d$.
For $i \in [d]$ we denote by $\mu_i$ the marginal density of the $i$-th component.
We say that $\mu$ has a {\em polynomially-bounded marginal density} if there is $M = \poly(d)$ such that for every $i \in [d]$ and $t \in [0,1]$ we have $\mu_i(t) \leq M$.

\textbf{Benign functions.} Below we formally define {\em benign functions}:
\begin{definition}
	We say that a function $f:[0,1]^d \rightarrow [0,1]$ is {\em benign} if it satisfies the following conditions:
	\begin{enumerate}
		\item  It is $\poly(d)$-Lipschitz.
		\item It is exponential-time computable: For every $c=\co(\log(d))$, there is an algorithm $\ca$ that runs in time exponential in $d$, such that for every input $\bx \in [0,1]^d$ where each component is given by a binary representation with $c$ bits, it returns $f(\bx)$ within precision of $c$ bits. Namely, the algorithm $\ca$ returns a $c$-bits binary representation of $\by \in [0,1]$ such that $|\by-f(\bx)| \leq \frac{1}{2^c}=\frac{1}{\poly(d)}$.
	\end{enumerate}
	We say that a benign function $f$ is {\em polynomial-time benign} (respectively, {\em polynomial-space benign}), if it is computable in polynomial time (respectively, space). Namely, for every $c=\co(\log(d))$, there is an algorithm $\ca$ that runs in polynomial time (respectively, space), such that for every input $\bx \in [0,1]^d$ where each component is given by $c$ bits, it returns $f(\bx)$ within precision of $c$ bits.
\end{definition}

In the above definition we use $c=\co(\log(d))$ bits since it corresponds to precision of $1/{\poly(d)}$.
Standard mathematical operations such as addition, multiplication, division, square root, $\exp$, $\log$, $\sin$ and $\cos$, can all be computed within precision of $m$ bits in $\poly(m)$ time, as well as combinations of such operations. 
\stam{
	Thus, the requirement that we can compute $f$ within precision of a logarithmic number of bits in exponential time is very mild. Moreover, even a requirement that such computation can be done in polynomial time is mild.
}%stam
Thus, the requirement that we can compute $f$ within precision of a logarithmic number of bits in polynomial time is standard.
Note that 
a function may be expressible by an exponential-size (or even polynomial-size) neural network but not exponential-time computable, since it may not be possible to find the neural network in exponential time.
Namely, expressiveness by bounded-size neural networks corresponds to \emph{non-uniform computability}, while time-computability is in the \emph{uniform} sense.

The Lipschitzness assumption is also standard, since in learning tasks it is usually sufficient to consider prediction functions with a bounded Lipschitz constant, as they tend to be more robust and not very sensitive to small changes in the input. E.g., if we slightly perturb the pixel values in some given image we usually do not expect the target distribution over the possible image labels to change dramatically. Note that the Lipschitzness assumption is w.r.t. the target function that we want to approximate, and the neural network that approximates the function is not limited in its Lipschitzness.
Moreover, target functions with very large Lipschitz constants are usually harder to learn with gradient-based methods.

Consider, for example, the function given in \cite{daniely2017depth} to establish depth-separation between depth~$2$ and~$3$. This function is of the form\footnote{The function in \cite{daniely2017depth} is defined on the sphere and in a slightly different way, but it is easily reduced to a function of this form on the unit ball (see \cite{vardi2020neural,safran2019depth}).} $f(\bx)=\sin(\frac{1}{2} \pi d^3 \norm{\bx})$, and hence it is clearly $\poly(d)$-Lipschitz and computable in polynomial time. Thus, it is polynomial-time benign.

\textbf{Neural networks.}
We consider feedforward neural	networks, computing functions from $\reals^d$ to $\reals$. 
A neural network is composed of layers of neurons, where each neuron, except for the output neuron, has an activation function $\sigma:\reals \rightarrow \reals$.
We focus on the ReLU activation, namely, $\sigma(z) = [z]_+ := \max\{0,z\}$. When we consider other activation functions we explicitly mention it. 
We define the \emph{depth} of the network as the number of layers. Denoting the number of neurons in the $i$-th layer by $n_i$, we define the {\em width} of a network as $\max_{i}n_i$, and the {\em size of the network} as $\sum_{i}n_i$. 
We sometimes consider networks with an activation function also in the output neuron, and networks with multiple outputs.

%stam

\textbf{Threshold circuits.}
A threshold circuit is a neural network with the following restrictions:
\begin{itemize}
	\item The activation function in all neurons is $\sigma(z) = \sign(z)$. We define $\sign(z)=0$ for $z \leq 0$, and $\sign(z)=1$ for $z > 0$. A neuron in a threshold circuit is called a {\em threshold gate}. 
	The function computed by a threshold gate is called {\em linear threshold function (LTF)}. 
	We sometimes denote this function by $L_{\ba=(a_1 \ldots a_m), \theta}$, where $\ba$ are the weights and $\theta$ is the bias term.
	%The \emph{total weight} of an LTF with weights $a_1,\dots, a_m$ is $\sum_{i=1}^m|a_i|$.
	\item The output gates also have a $\sign$ activation function. Hence, the output is binary.
	\item We always assume that the input to a threshold circuit is a binary vector $\bx \in \{0,1\}^d$.
	\item Since every threshold circuit with real weights can be expressed by a threshold circuit of the same size with integer weights bounded by $2^{\co(d \log(d))}$ (cf. \cite{goldmann1998simulating}), we assume w.l.o.g. that all weights are integers that can be represented by $\poly(d)$ bits.
\end{itemize}
We denote by $TC^0$ the class of polynomial-sized threshold circuits of constant depth, and by $TC^0_k$ the class of polynomial-sized threshold circuits of depth $k$.

\textbf{Functions approximation.}
We say that a function $f$ can be {\em approximated by a $\poly(d)$-sized neural network of depth $k$} (with respect to a distribution $\cd$) if for every $\epsilon=\frac{1}{\poly(d)}$ we have 
%$\E_{\bx \sim \cd}(f(\bx)-N(\bx))^2 \leq \epsilon$ 
$\snorm{f-N}_{L_2(\cd)} \leq \epsilon$ 
for some depth-$k$ network $N$ of size $\poly(d)$.
For $\epsilon=\epsilon(d)$ and $m=m(d)$,
we say that $f$ can be {\em $\epsilon$-approximated by a neural network of size $m$} (with respect to a distribution $\cd$) if 
%$\E_{\bx \sim \cd}(f(\bx)-N(\bx))^2 \leq \epsilon$ 
$\snorm{f-N}_{L_2(\cd)} \leq \epsilon$
for some network $N$ of size $m$.
While we focus on approximation in the $L_2$ sense, we note that our results apply also to $L_p$ for every $1 \leq p < \infty$.
For a function $f:[0,1]^d \rightarrow [0,1]$ and a neural network $N$, we denote $\norm{f-N}_\infty=\sup_{\bx \in [0,1]^d}|f(\bx)-N(\bx)|$.

\textbf{Depth-separation and size-separation.}
We say that there is {\em depth-separation} between networks of depth $k$ and depth $k'$ for some integers $k'>k$, if there is a distribution $\cd$ on $[0,1]^d$ and a function $f:[0,1]^d \rightarrow [0,1]$ that can be approximated (with respect to $\cd$) by a $\poly(d)$-sized neural network of depth $k'$ but cannot be approximated by $\poly(d)$-sized networks of depth $k$.
We note that our definition of depth-separation is a bit weaker than most existing depth-separation results, which actually show difficulty of approximation even up to constant accuracy (and not just $1/\poly(d)$ accuracy). However, depth separation in that sense implies depth separation in our sense. Hence, the barriers we show here to depth separation imply similar barriers under this other (or any stronger) notion of depth separation. Our definition is similar to the definition in \cite{vardi2020neural}.

We say that there is {\em size-separation} between networks of size $\co(m)$ and size $\co(m')$, if there is a distribution $\cd$ on $[0,1]^d$, a function $f:[0,1]^d \rightarrow[0,1]$, and $\epsilon=\frac{1}{\poly(d)}$, such that $f$ can be $\epsilon$-approximated (with respect to $\cd$) by a neural network of size $\co(m')$ but cannot be $\epsilon$-approximated by networks of size $\co(m)$.
Thus, networks of size $\co(m')$ are more powerful than networks of size $\co(m)$ in their ability to approximate $f$ within some reasonable accuracy $\epsilon$.
\stam{
Some stronger notions of size separations may also be considered, and we note that the barriers we show here to size separation imply similar barriers under stronger notions. 
Our positive results for size separation 
hold already for a constant $\epsilon$.
}%stam

\textbf{Natural-proofs barrier.}
The study of circuit lower bounds is a central challenge in theoretical computer science, but despite many attempts the results in this field are limited \citep{arora2009computational}.
In a seminal work, \cite{razborov1997natural} described a main technical limitation of current approaches for proving circuit lower bounds: They defined a notion of ``natural proofs" for a circuit lower bound (which include current proof techniques), and showed that obtaining lower bounds with such proof techniques would violate a widely accepted conjecture, namely,
that pseudorandom functions exist. This natural-proofs barrier (partially) explains the lack of progress on circuit lower bounds.
More formally, they show that if a class $\cc$ of circuits contains a family of pseudorandom functions, then showing for some function $f$ that $f \not \in \cc$ cannot be done with a natural proof. 

\section{Barriers to depth lower bounds and to depth-separation}
\label{sec:barriers depth}

\cite{telgarsky2016benefits} showed that there exists a family of univariate functions $\{\varphi_k\}_{k=1}^\infty$ on the interval $[0,1]$, such that the function $\varphi_k$ is $2^k$-Lipschitz, it can be expressed by a network of depth $k$ and width $\co(1)$, but it cannot be $\epsilon$-approximated (for some constant $\epsilon$) by any $o(k/\log(k))$-depth, $\poly(k)$-width network with respect to the uniform distribution on $[0,1]$.
The function $\varphi_k$ consists of $2^{k-1}$ identical triangles of height $1$\footnote{The function $\varphi_k$ is obtained by composing the function $z \mapsto [2z]_+ - [4z-2]_+$ with itself $k$ times.}.
Consider the functions $\{f_d\}_{d=1}^\infty$ where $f_d:[0,1]^d \rightarrow [0,1]$ is such that $f_d(\bx) = \varphi_d(x_1)$. Thus, $f_d$ depends only on the first component of $\bx$.
The result of \cite{telgarsky2016benefits} implies that the function $f_d$ can be expressed by a network of width $\co(1)$ and depth $d$, but cannot be approximated by networks of width $\poly(d)$ and constant depth w.r.t. the uniform distribution on $[0,1]^d$. Hence, there is separation between constant and nonconstant depths. Namely, there are functions that can be computed by a neural network of $\poly(d)$ size, but cannot be approximated by networks of $\poly(d)$ size and constant depth. 
Note that $f_d$ is $2^d$-Lipschitz.

As we discussed in the introduction, the main weakness of the above result, is that the Lipschitzness of $f_d$ is superpolynomial. 
Hence, an interesting question is whether such a result can be obtained for benign functions. 
The following theorem implies barriers to depth-lower-bounds and to depth-separation.

\begin{theorem}
\label{thm:bounded depth}
	If there exists a benign function $f:[0,1]^d \rightarrow [0,1]$, that cannot be approximated by a neural network of size $\poly(d)$ and constant depth $k \geq 4$, w.r.t. a distribution $\mu$ with a polynomially-bounded marginal density, then $\EXP \not \subseteq \TC^0_{k-2}$. 
\end{theorem}
\begin{proof}[Proof idea (for complete proof see Appendix~\ref{app:proof of theorem bounded depth})]
	Let $c=\co(\log(d))$. Since $f$ is exponential-time computable then there is an exponential-time algorithm $\ca$, such that for an input $\bx$ given by $c$ bits for each component, it returns $f(x)$ within precision of $c$ bits. Let $\hat{f}:\{0,1\}^{c \cdot d} \rightarrow \{0,1\}^c$ be the function that $\ca$ computes. Assume that $\hat{f}$ can be computed by a threshold circuit $T$ of size $\poly(d)$ and depth $k-2$. We construct a neural network $N$ of size $\poly(d)$ and depth $k$ that approximates $f$ and thus reach a contradiction. It implies that the function $\hat{f}$ can be computed in exponential time but cannot be computed by a threshold circuit of size $\poly(d)$ and depth $k-2$, and hence we are able to obtain $\EXP \not \subseteq \TC^0_{k-2}$.
	
	The network $N$ first transforms w.h.p. over $\bx \sim \mu$ the input $\bx \in [0,1]^d$ to a binary representation $\hbx \in \{0,1\}^{c \cdot d}$, then it simulates the threshold circuit $T$ to obtain $T(\hbx)=\hat{f}(\hbx)$, and finally it converts the output of $T$ from a binary representation to the corresponding real value. Note that since $f$ is $\poly(d)$-Lipschitz, then for an appropriate $c$, transforming the input to the binary representation does not hurt the approximation too much.
\end{proof}

\begin{remark}[Barrier to depth lower bounds]
	It is a longstanding open problem whether $\EXP \not \subseteq \TC^0_{2}$ (and even whether $\NEXP \not \subseteq \TC^0_{2}$)  \citep{razborov1992small,oliveira2015unconditional,chen2018toward}. 
	By Theorem~\ref{thm:bounded depth},	proving existence of a benign function that cannot be approximated by a network of depth $k \geq 4$ would imply that $\EXP \not \subseteq \TC^0_{k-2}$ and thus solve this open problem.
	%Moreover, by similar arguments we can also conclude that proving existence of a benign function that cannot be approximated by any network of polynomial size and constant depth, would solve the open problem of whether $\EXP \not \subseteq \TC^0$. 
\end{remark}

\begin{remark}[A stronger barrier for ``more benign" functions]
	For polynomial-time benign functions and polynomial-space benign functions, 
	we can obtain even stronger barriers to depth lower bounds: Proving existence of a polynomial-time (respectively,  polynomial-space) benign function that cannot be approximated by a network of depth $k \geq 4$, would imply that $\Ptime \not \subseteq \TC^0_{k-2}$ (respectively, $\PSPACE \not \subseteq \TC^0_{k-2}$). 
	Since $\Ptime \subseteq \Ppoly$, establishing $\Ptime \not \subseteq \TC^0_{k-2}$ would also imply that $\TC^0_{k-2} \neq \Ppoly$. 
	Moreover, we can also conclude that proving existence of a polynomial-time benign function that cannot be approximated by networks of polynomial size and any constant depth, would solve the open problem of whether $\TC^0 \neq \Ppoly$. 
\end{remark}

\begin{remark}[Natural-proofs barrier for $k \geq 6$]
\label{rem:natural proof}
	\cite{naor2004number} and \cite{krause2001pseudorandom} showed a candidate pseudorandom function family in $\TC^0_4$. By \cite{razborov1997natural}, it implies that there is a natural-proofs barrier to proving circuit lower bounds for threshold circuits of depth at least $4$. Since by 
	Theorem~\ref{thm:bounded depth} 
	proving existence of a benign function that cannot be approximated by networks of depth $k \geq 6$ would imply a lower bound for threshold circuits of depth $k-2 \geq 4$, then showing such depth lower bounds for networks of depth at least $6$ would need to overcome the natural-proofs barrier.	
\end{remark}

\begin{remark}[Barrier to depth separation]
	Theorem~\ref{thm:bounded depth}	clearly implies a barrier to showing depth-separation results for benign functions, since such results would imply that there exists a benign function that cannot be approximated by a network of some bounded depth. Thus, there is a barrier already to showing depth separation for a benign function between $\poly(d)$-sized networks of depth $4$ and $\poly(d)$-sized networks of 
	%polynomial depth. 
	unbounded depth.
	We note that \cite{vardi2020neural} established a complexity-theoretic barrier to showing depth-separation, but their result applies only to separation between two constant depths.
\end{remark}
 
By its definition, the benign function $f$ in Theorem~\ref{thm:bounded depth} satisfies two requirements: $\poly(d)$-Lipschitzness and exponential-time computability.
As we already discussed, the construction of \cite{telgarsky2016benefits} gives a function that cannot be approximated by $\poly(d)$-sized networks of any constant depth. This function is efficiently computable, but is not $\poly(d)$-Lipschitz.
We now show that a similar result can be obtained with a function that is $1$-Lipschitz, but we do not have any guarantees on its computability.

\begin{theorem}
\label{thm:counting}
	There exists a $1$-Lipschitz function $f:[0,1]^d \rightarrow [0,1]$ that cannot be approximated w.r.t. a distribution with a polynomially-bounded marginal density, by a $\poly(d)$-sized neural network whose weights are represented by a $\poly(d)$ number of bits.
\end{theorem}

We note that Theorem~\ref{thm:bounded depth} holds already for networks whose weights are represented by a $\poly(d)$ number of bits. That is, if the benign function $f$ cannot be approximated by a network of size $\poly(d)$ and depth $k$, whose weights are represented by $\poly(d)$ bits, then $\EXP \not \subseteq \TC^0_{k-2}$. Hence the barriers to depth lower bounds (for benign functions) hold already for the network considered in Theorem~\ref{thm:counting}.
The proof of Theorem~\ref{thm:counting} follows essentially from a counting argument: we show that in order to cover the set of all $1$-Lipschitz functions $f:[0,1]^d \rightarrow [0,1]$ with balls of radius $\epsilon=\frac{1}{\poly(d)}$ (w.r.t. norm $L_2(\mu)$), the number of balls required is larger than the number of $\poly(d)$-sized networks whose weights are represented by $\poly(d)$ bits. See Appendix~\ref{app:proof of theorem counting} for the proof.

\section{Barriers to size lower bounds and to size separation}
\label{sec:barriers size}

\subsection{Barriers to superpolynomial lower bounds}

In Section~\ref{sec:barriers depth} we studied which functions cannot be approximated by neural networks of polynomial size and bounded depth. Here, we study which functions cannot be approximated by neural networks of polynomial size without restricting its depth. The barriers from Section~\ref{sec:barriers depth} clearly apply also to networks of polynomial size and unbounded depth. 
Indeed, if a benign function cannot be approximated by any polynomial-sized network, then it clearly cannot be approximated by any polynomial-sized network of a bounded depth.
Thus, there is a barrier to proving 
superpolynomial size lower bounds for benign functions.
We now show that a barrier can be obtained already for functions that satisfy a weaker requirement.
\begin{definition}
We say that a function $f:[0,1]^d \rightarrow [0,1]$ is {\em semi-benign} if it satisfies the following conditions:
\begin{enumerate}
	\item It is $2^{\poly(d)}$-Lipschitz.
	\item It is exponential-time computable for $\poly(d)$-bits inputs: For every  $c = \poly(d)$ and $c'=\co(\log(d))$, there is an  algorithm $\ca$ that runs in time exponential in $d$, such that for every input $\bx \in [0,1]^d$ where each component is given by a binary representation with $c$ bits, it returns $f(\bx)$ within precision of $c'$ bits.
	Namely, the algorithm $\ca$ returns a $c'$-bits binary representation of $\by \in [0,1]$ such that $|\by-f(\bx)| \leq \frac{1}{2^{c'}} = \frac{1}{\poly(d)}$.
\end{enumerate}
\end{definition}	
Note that, unlike benign functions, the Lipschitz constant of semi-benign functions can be exponential in $d$.
We show a barrier to size lower bounds for semi-benign functions.

\begin{theorem}
\label{thm:bounded size}
	If there exists a semi-benign function $f:[0,1]^d \rightarrow [0,1]$, that cannot be approximated by neural networks of size $\poly(d)$ w.r.t. a distribution $\mu$ with a polynomially-bounded marginal density, then $\EXP \not \subseteq \Ppoly$.
\end{theorem}
The proof of the theorem follows roughly a similar idea to the proof of Theorem~\ref{thm:bounded depth}. However, since the Lipschitz constant of $f$ is exponential, then we need to use a binary representation with a polynomial number of bits. Transforming w.h.p. the input $\bx \in [0,1]^d$ to such a representation can be done with a $\poly(d)$-sized network using the construction of \cite{telgarsky2016benefits}. See Appendix~\ref{app:proof of theorem bounded size} for the complete proof.

\begin{remark}[Barrier to superpolynomial size lower bounds]
	It is a longstanding open problem whether $\EXP \not \subseteq \Ppoly$. Also, as we discussed in Remark~\ref{rem:natural proof}, there is a natural-proofs barrier to solving this problem. 
	Hence, Theorem~\ref{thm:bounded size} implies a barrier to showing that there exists a semi-benign function that cannot be approximated by polynomial-sized networks.
\end{remark}

Recall that by \cite{telgarsky2016benefits}, there exists a family of univariate functions $\{\varphi_k\}_{k=1}^\infty$ on the interval $[0,1]$, such that the function $\varphi_k$ is $2^k$-Lipschitz, 
%can be expressed by a network of depth $k$ and width $\co(1)$, but 
and it cannot be approximated by any $o(k/\log(k))$-depth, $\poly(k)$-width network with respect to the uniform distribution on $[0,1]$.
Consider the functions $\{f_d\}_{d=1}^\infty$ where $f_d:[0,1]^d \rightarrow [0,1]$ is such that $f_d(\bx) = \varphi_{d^{\log(d)}}(x_1)$. The result of \cite{telgarsky2016benefits} implies that $f_d$ cannot be approximated by a network of depth $\poly(d)$ and width $\poly(d)$ w.r.t. the uniform distribution on $[0,1]^d$. Hence, it cannot be approximated by any network of size $\poly(d)$. The function $f_d$ is exponential-time computable for $\poly(d)$-bits inputs. However, note that $f_d$ is $2^{d^{\log(d)}}$-Lipschitz. Thus, its Lipschitz constant is super-exponential and hence the function is not semi-benign.

We note that the barrier implied by Theorem~\ref{thm:bounded size} holds already for networks of size $\poly(d)$ whose weights are represented by a $\poly(d)$ number of bits.
By Theorem~\ref{thm:counting} there is a $1$-Lipschitz function that cannot be approximated by a $\poly(d)$-sized network whose weights have $\poly(d)$ bits. However, we do not have any guarantees on the time complexity of computing this function.

\subsection{Barriers to $\omega(d \log^2(d))$ lower bounds}

Here, we establish a barrier to showing $\omega(d \log^2(d))$-size lower bounds with polynomial-time benign functions. 
The proof of the theorem uses ideas from the proofs of Theorems~\ref{thm:bounded depth} and~\ref{thm:bounded size}, and is given in Appendix~\ref{app:proof of theorem bounded size2}. 
\begin{theorem}
\label{thm:bounded size2}
	If there exist a polynomial-time benign function $f:[0,1]^d \rightarrow [0,1]$, a distribution $\mu$ with a polynomially-bounded marginal density, and $\epsilon=\frac{1}{\poly(d)}$, such that $f$ cannot be $\epsilon$-approximated by neural networks of size $\co(d \log^2(d))$ w.r.t. $\mu$, then there is a function $g:\{0,1\}^{d'} \rightarrow \{0,1\}$ in $\Ptime$ that cannot be computed by threshold circuits of size $\co(d')$.
\end{theorem}

\begin{remark}[Barrier to $\omega(d \log^2(d))$-size lower bounds and to size separation]
\label{rem:barrier to size separation}
	It is a longstanding open problem whether there is a function in $\Ptime$ (or even in $\NP$) that cannot be computed by threshold circuits (or even Boolean circuits) of linear size (cf. \cite{find2016better, arora2009computational}). Hence, Theorem~\ref{thm:bounded size2} implies a barrier to proving that there exists a polynomial-time benign function that cannot be approximated by networks of size $\co(d \log^2(d))$. 
	Thus, it also implies a barrier to showing size-separation results for polynomial-time benign function, between size $\co(d \log^2(d))$ and some larger size.
\end{remark}

Let $\{\varphi_k\}_{k=1}^\infty$ be the functions from  \cite{telgarsky2016benefits}, and recall that $\varphi_k$ is $2^k$-Lipschitz, and that it cannot be approximated by any $o(k/\log(k))$-depth, $\poly(k)$-width network with respect to the uniform distribution on $[0,1]$.
Consider the functions $\{f_d\}_{d=1}^\infty$ where $f_d:[0,1]^d \rightarrow [0,1]$ is such that $f_d(\bx) = \varphi_{d \log^4(d)}(x_1)$. Thus, $f_d$ cannot be approximated by networks of depth $\co(d \log^2(d))$ and width $\poly(d)$ w.r.t. the uniform distribution on $[0,1]^d$. Therefore, it cannot be approximated by any network of size $\co(d \log^2(d))$. 
The function $f_d$ is polynomial-time computable. 
However, note that $f_d$ is $2^{d \log^4(d)}$-Lipschitz. Thus, its Lipschitz constant is superpolynomial, and hence the function is not benign.

Finally, we note that the barrier implied by Theorem~\ref{thm:bounded size2} holds already for networks of size $\co(d \log^2(d))$ whose weights are represented by a $\poly(d)$ number of bits.
By Theorem~\ref{thm:counting} there is a $1$-Lipschitz function that cannot be approximated by such networks. However, we do not have any guarantees on the time complexity of computing this function.

\section{Lower bounds for Boolean functions}
\label{sec:boolean}

In this section we establish lower bounds for the size needed to implement certain explicit Boolean functions with neural networks. 
Our lower bounds are with respect to neural networks that exactly interpolate a Boolean function $f$. Namely, for every Boolean input the network outputs the \emph{exact} Boolean value of $f$. The same lower bounds (with nearly identical proofs) apply also to neural networks where the output neuron has a threshold activation function. Such thresholding is often used when considering Boolean functions implemented by neural networks  (e.g., \cite{mukherjee2017lower,martens2013representational,maass1997bounds,koiran1996vc,maass1991computational}).   

\subsection{$\Omega(d/\log d)$ lower bound for approximation in $L_2(\cu(\{0,1\}^d))$}

The lower bounds in this section are based on communication-complexity lower bounds.
We assume familiarity with communication complexity (for an excellent introduction see \cite{kushilevitz1997communication}), 
%and follow the standard definitions in \cite{kushilevitz1997communication}.
and consider the worst-case partition setting 
%in communication complexity 
(cf. Chapter~7 in \cite{kushilevitz1997communication}). In this setting there is a Boolean function $f$ with $d$ inputs.
An input $(y_1,...,y_d)$ is partitioned between two players, Alice and Bob, with unbounded computational power. In other words, Alice has the set of bits $\{y_i|i \in I\}$
and Bob has the set of bits $\{y_j|j \in [d]\setminus I\}$ (where $I$ is a nonempty subset of $[d]$).
The goal of Alice and Bob is to compute $f(y_1, \ldots ,y_d)$ using a predefined protocol with as few bits exchanged between the players.
In every round of the protocol a single player can send one bit of communication to the other party. The \emph{cost} of a communication protocol on a given input and partition is the number of bits exchanged between the players.
The cost of a given protocol is its maximum cost over all possible inputs and partitions.
We consider randomized protocols, where Alice and Bob can use random bits and furthermore have access to a common source of random bits.
The randomized (worst case) communication complexity of $f$, denoted by $R(f)$, is the minimal cost of a protocol that results with computing $f$
correctly with probability at least $2/3$ on every possible input 
%the random string used by the players and the partition between the players. 
and partition.

\cite{nisan1993communication} observed that a threshold circuit $C$ for a Boolean function $f$ can be used for a communication protocol evaluating $f$, whose cost (number of bits exchanged) is not much larger than the size of $C$, as any threshold gate can be evaluated with a protocol of logarithmic cost (in the number of inputs to the gate). 
Therefore, lower bounds on the communication complexity of $f$ imply lower bounds on the threshold circuit complexity of $f$. When trying to use this idea for neural networks, a difficulty is that the outputs of the neurons are real numbers, as opposed to the case of threshold circuits where the output of every gate is Boolean. We circumvent this problem by noticing that for two parties who have the parameters of a ReLU network computing a Boolean function, and want to determine the sign of the output for a given Boolean input,
%(namely, whether the output is $0$ or $1$), 
the parties can determine the sign of the output of each neuron in the network recursively by a low-communication protocol. That is, once the players know for all neurons in layers $1,\ldots,j-1$ whether their outputs are zero or positive, we show that they can determine for neurons in layer $j$ whether their outputs are zero of positive, while using a logarithmic number of bits for each neuron.
Indeed, they can remove the neurons with output $0$ in layers $1,\ldots,j-1$, and then the input to each neuron in layer $j$ is a linear function of the network's inputs, which implies that the sign of the neuron's output can be determined with the protocol of \cite{nisan1993communication}.
This idea is formalized in the following theorem, and extended to the more general case of neural networks with $k$-piecewise-linear activation functions. See Appendix~\ref{app:proof of theorem cc ReLU} for a proof.

\begin{theorem}
\label{thm:CC_ReLU}
    Let $h:\{0,1\}^d \rightarrow \{0,1\}$ be such that $R(h)=\Omega(d)$.
	Any ReLU network computing $h$ has size $\Omega(d/\log d)$. More generally, any neural network with a $k$-piecewise-linear activation function computing $h$ has size $\Omega(d/(\log d \cdot \log k))$.
\end{theorem}

Two classical polynomial-time computable Boolean functions, 
%with $d$ inputs 
that have 
$\Omega(d)$ randomized communication complexity are \emph{disjointness} and \emph{inner product}.  
The disjointness function $f=\text{DISJ}_d$ evaluates to $1$ on Boolean inputs $(x_1,x_2 \ldots x_{2d})$ iff the two subsets of $[d]$ whose characteristic vectors are $(x_1,x_2, \ldots, x_{d}), (x_{d+1},\ldots, x_{2d})$ are disjoint. In other words, $f$ evaluates to $1$ iff there is no index $j \in [d]$ where $x_j=x_{d+j}=1$. It is known that the randomized communication complexity of the disjointness function is $\Omega(d)$ \citep{kalyanasundaram1992probabilistic,razborov1992distributional,bar2004information}.
The inner product function $g=\text{IP}_d$ with Boolean inputs $(x_1,x_2, \ldots, x_{d},y_1,y_2, \ldots, y_{d})$ evaluates to $\sum_{i=1}^d x_i y_i \mod 2$. The inner product function is known to satisfy
$R(g)=\Omega(d)$ as well \citep{babai1986complexity}. 
Thus, Theorem~\ref{thm:CC_ReLU} implies the following corollary:

\begin{corollary}
\label{cor:CC_ReLU}
	Any ReLU network computing $\text{DISJ}_d$ or $\text{IP}_d$ has size $\Omega(d/\log d)$.	
\end{corollary}

In the next section we will improve Corollary~\ref{cor:CC_ReLU} and establish a linear lower bound. However, since the results here are based on randomized communication complexity, then we are able to obtain a $\Omega(d/\log d)$ lower bound already for approximation in the $L_2$ sense. 
Intuitively, it follows from the following argument. Assume that there is a neural network $N$ that approximates $\text{IP}_d$ in the $L_2$ sense w.r.t. the uniform distribution over $\{0,1\}^{2d}$. It implies that there is a neural network $N'$ that computes $\text{IP}_d$ correctly on a large fraction of the inputs. %Hence, there is a randomized protocol that for every $\bx,\by \in \{0,1\}^d$ computes $\text{IP}_d(\bx,\by)$ w.h.p. as follows. 
Using $N'$ and the protocol from Theorem~\ref{thm:CC_ReLU}, we show a randomized protocol, that for random vectors $\bx',\by' \in \{0,1\}^d$, computes w.h.p. $\text{IP}_d(\bx+\bx',\by+\by'), \text{IP}_d(\bx+\bx',\by'), \text{IP}_d(\bx',\by+\by'), \text{IP}_d(\bx',\by')$. Then, we have
\[
	\text{IP}_d(\bx,\by) = \text{IP}_d(\bx+\bx',\by+\by') +  \text{IP}_d(\bx+\bx',\by') + \text{IP}_d(\bx',\by+\by') + \text{IP}_d(\bx',\by') \mod 2~.
\]
Thus, the protocol computes w.h.p. $\text{IP}_d(\bx,\by)$ for every $\bx,\by$. Therefore, we show that the lower bound on the randomized communication complexity of $\text{IP}_d(\bx,\by)$ implies a lower bound on the sizes of $N'$ and $N$.  
This idea is formalized in the following theorem (see Appendix~\ref{app:proof of theorem IP approx} for a proof).

\begin{theorem}
\label{thm:IP approx}
	Let $\epsilon=\frac{1}{20}$.
	Let $N$ be a ReLU network that $\epsilon$-approximates the function $\text{IP}_d$ w.r.t. the uniform distribution over $\{0,1\}^{2d}$. Then, $N$ has size $\Omega(d/\log d)$.
\end{theorem}

\subsection{$\Omega(d)$ lower bound for exact computation} 

We now utilize the real communication model introduced by \cite{kraivcek1998interpolation} (see also \cite{de2016limited}) in order to establish linear size lower bounds for neural networks.
Consider a Boolean function $f:\{0,1\}^d \rightarrow \{0,1\}$ whose input is split between two players, Alice and Bob. 
We define the following real (deterministic) communication protocol. In each round, each player outputs a real number, based on its input ($\bx$ for Alice and $\by$ for Bob) and a word $\bw \in \{0,1\}^*$ (accessible to both players) defined as follows. Before the first round, $w$ is the empty word. At round $i$ Alice outputs $\alpha \in \mathbb{R}$ and Bob outputs $\beta \in \mathbb{R}$. A referee receives $\alpha, \beta$ and alters the word $\bw$ to $\bw 1$ if $\alpha > \beta$ and to $\bw 0$ if $\alpha \leq \beta$. The \emph{cost} of the protocol with respect to a given input to $f$ and its bipartition is the final length of $\bw$, and the cost of an arbitrary protocol is the maximal 
%number of rounds 
cost
for every input and bipartition. 
The real communication complexity of a function $f$, denoted by $CC^{\mathbb{R}}(f)$, is the minimal cost of a protocol, such that when the protocol halts, $f$ can be computed (deterministically with zero error) from the word $\bw$ attained at the termination of the protocol. This model is equivalent to a communication protocol with an access to a greater-than (GT) oracle (see~\cite{chattopadhyay2019equality} for details).

It can be easily shown that the real communication complexity of a linear threshold function is $1$. Using a similar reasoning to the proof of Theorem~\ref{thm:CC_ReLU}, we show that lower bounds on the real communication complexity of a function $f$ imply lower bounds on the size of a neural network that computes $f$.
Indeed, suppose that a Boolean function $f$ can be computed by a ReLU network $N$, then we have a real communication protocol that computes $f$ by evaluating the signs of the outputs of neurons in $N$ (including the output neuron) recursively, and has cost linear in the size of $N$. The argument clearly holds also for threshold circuits, and can easily be extended to neural networks with a $k$-piecewise-linear activation.
Formally, we have the following theorem (see Appendix~\ref{app:proof of theorem CC real} for a proof).
\begin{theorem}
\label{thm:CC_real}
    Let $f:\{0,1\}^d \rightarrow \{0,1\}$ be such that $CC^{\mathbb{R}}(f)=\Omega(d)$.
    Any ReLU network or threshold circuit computing $f$ has size $\Omega(d)$.
	Any neural network with a $k$-piecewise-linear activation function computing $f$ has size $\Omega(d/\log k)$. 
\end{theorem}

By Lemma~4.9 in \cite{chattopadhyay2019equality}, we have 
$CC^{\mathbb{R}}(\text{DISJ}_d)=\Omega(d)$.
A similar lower bound for the inner-product function is known to experts, but we are not aware of a previous proof. In Appendix~\ref{app:IP real bound} we give a proof of this fact based on \cite{chattopadhyay2019equality}. Thus, we have the following corollary:
\begin{corollary}
\label{cor:CC_real}
	Any ReLU network or threshold circuit computing $\text{DISJ}_d$ or $\text{IP}_d$ has size $\Omega(d)$.
	Any neural network with a $k$-piecewise-linear activation function computing $\text{DISJ}_d$ or $\text{IP}_d$ has size $\Omega(d/\log k)$.
\end{corollary}

We note that for the case of computing $\text{IP}_d$ with threshold circuits, the above result was already shown with different methods \citep{groeger1993linear,jukna2012boolean,roychowdhury1994lower}. The first proof of this linear lower bound in \cite{groeger1993linear} is based on a gate elimination argument, whereas the proof of \cite{roychowdhury1994lower} uses combinatorial properties of communication matrices of threshold circuits. 
However, Corollary~\ref{cor:CC_real} is the first linear lower bound for computing $\text{DISJ}_d$ with threshold circuits, and for computing $\text{DISJ}_d$ or $\text{IP}_d$ with neural networks.

\subsection{Linear upper bounds}

Recall that by Corollary~\ref{cor:CC_real}, $\text{DISJ}_d$ and $\text{IP}_d$ cannot be computed by a neural network of size $o(d)$, and by Theorem~\ref{thm:IP approx}, $\text{IP}_d$ cannot be approximated in the $L_2$ sense by a network of size $o(d/\log d)$. In the following theorem we show a linear upper bound for computing $\text{DISJ}_d$ and $\text{IP}_d$. The theorem follows by straightforward constructions (see Appendix~\ref{app:proof of theorem boolean upper bound} for a proof).

\begin{theorem}
\label{thm:boolean upper bound}
	The functions $\text{DISJ}_d$ and $\text{IP}_d$ can be computed by ReLU networks of size $\co(d)$.
\end{theorem}

\section{Size separation for benign functions}
\label{sec:separation real}

We utilize our results on Boolean functions in order to establish size separation for polynomial-time benign functions.
The following proposition allows us to 
translate our results from the Boolean setting to the 
continuous setting (see proof in Appendix~\ref{app:proof of prop from bool to real}).

\begin{proposition}
\label{prop:from bool to real}
	Let $g:\{0,1\}^d \rightarrow \{0,1\}$. There is a $4$-Lipschitz function $f:[0,1]^d \rightarrow [0,1]$ that agrees with $g$ on $\{0,1\}^d$, and a distribution $\mu$ on $[0,1]^d$ with a polynomially-bounded marginal density, such that:
	\begin{enumerate}
		\item If $g$ cannot be $\epsilon$-approximated by neural networks of size $\co(m)$ with respect to the uniform distribution over $\{0,1\}^d$, then $f$ cannot be $\epsilon$-approximated by neural networks of size $\co(m)$ with respect to $\mu$.
		\item If $g$ can be computed by a neural network of size $\tilde{m}$, then there is a neural network $\tilde{N}$ of size $\tilde{m} + 2d$ such that 
		$\snorm{\tilde{N}-f}_{L_2(\mu)}=0$.
		\item If $g \in \Ptime$ then $f$ can be computed in polynomial time.
	\end{enumerate}
\end{proposition}

By Theorem~\ref{thm:IP approx}, the function $\text{IP}_d$ cannot be $\frac{1}{20}$-approximated w.r.t. the uniform distribution over $\{0,1\}^{2d}$ by neural networks of size $o(d/\log d)$. 
By Theorem~\ref{thm:boolean upper bound}, $\text{IP}_d$ can be computed by a network of size $\co(d)$. Also, $\text{IP}_d$ is clearly in $\Ptime$. Combining these results with Proposition~\ref{prop:from bool to real}, we obtain size-separation between networks of size $o(d/\log d)$ and size $\co(d)$:
\begin{corollary}
\label{cor:separation benign}
	There is a polynomial-time benign function $f:[0,1]^d \rightarrow [0,1]$ and a distribution $\mu$ on $[0,1]^d$ with a polynomially-bounded marginal density, such that:
	\begin{itemize}
		\item The function $f$ cannot be $\frac{1}{20}$-approximated by neural networks of size $o(d/\log d)$ w.r.t. $\mu$.
		\item There is a neural network $\tilde{N}$ of size $\co(d)$ such that $\snorm{\tilde{N}-f}_{L_2(\mu)}=0$. 
	\end{itemize} 
\end{corollary}

Recall that by Remark~\ref{rem:barrier to size separation}, there is a barrier to showing size separation for polynomial-time benign functions, between size $\co(d \log^2(d))$ and some larger size. Closing the gap between the above size-separation result and the barrier is an interesting topic for future research.

Finally, our lower bounds for exact computation of Boolean functions (Corollary~\ref{cor:CC_real}) allow us to obtain size separation also in the $L_\infty$ sense. Namely, we show a polynomial-time benign function $f:[0,1]^d \rightarrow [0,1]$ that can be computed by a network of size $\co(d)$, but cannot be approximated in the $L_\infty$ sense by networks of size $o(d)$. The function $f$ is the function computed by the neural networks from Theorem~\ref{thm:boolean upper bound}. 

\begin{theorem}
\label{thm:separation benign infty}
    There is a polynomial-time benign function $f:[0,1]^d \rightarrow [0,1]$ that can be computed by a neural network of size $\co(d)$, and for every neural network $N$ such that $\norm{f-N}_\infty \leq \frac{1}{3}$, the size of $N$ is $\Omega(d)$. 
\end{theorem}

We prove the theorem in Appendix~\ref{app:proof of theorem separation benign infty}.
Note that the barrier to size separation from Remark~\ref{rem:barrier to size separation} does not apply to approximation in the $L_\infty$ sense. 

\section*{Acknowledgements}

We would like to thank Mika Goos, Thomas Watson, Sasha Golovnev,
Arkadev Chattopadhyay, Pritish Kamath, Suhail Sherif and T.S. Jayram for useful discussions. 
This research is supported in part by European Research Council (ERC) grant 754705.

\bibliographystyle{abbrvnat}
\bibliography{bib}

\appendix

\section{Proofs for Section~\ref{sec:barriers depth}}

%\section{Proof of Theorem~\ref{thm:bounded depth}}
\subsection{Proof of Theorem~\ref{thm:bounded depth}}
\label{app:proof of theorem bounded depth}

Let $f:[0,1]^d \rightarrow [0,1]$ be a benign function. Assume that $f$ is $L$-Lipschitz for $L=\poly(d) \geq 1$.
Let $\epsilon=\frac{1}{\poly(d)}$, and assume that for every neural network $\cn$ of size $\poly(d)$ and depth $k$ we have $\snorm{f-\cn}_{L_2(\mu)} > \epsilon$ (for some $d$).
 Let $p(d)=\frac{4 L \sqrt{d}}{\epsilon}$, and let $c = \ceil{\log(p(d))}$. Thus, $2^c \geq p(d)$.
Let $\ci = \{\frac{j}{2^c}: 0 \leq j \leq  2^c - 1, j \in \integers\}$. 
For $\tx \in \ci$ we denote by $\bin(\tx) \in \{0,1\}^c$ the binary representation of $0 \leq j \leq 2^c-1$ such that $\tx=\frac{j}{2^c}$. For $\tbx \in \ci^d$ we denote by $\bin(\tbx) \in \{0,1\}^{c \cdot d}$ the concatenation of $\bin(\tx_i)$ for $i=1,\ldots,d$.
For $\hbx \in \{0,1\}^c$ we denote $\real(\hbx) = \frac{j}{2^c} \in \ci$, where $j$ is the integer whose binary representation is $\hbx$. 
%If $\hbx$ is a binary representation of an integer greater than $p(d)$ then we define $\real(\hbx)=0$. 
For $\hbx \in \{0,1\}^{c \cdot d}$ we denote $\real(\hbx) \in \ci^d$, such that the $i$-th component of $\real(\hbx)$ is $\real(\hx_{(i-1) \cdot c + 1},\ldots,\hx_{i \cdot c})$.
Finally, for $x \in [0,1]$ we denote by $\trunc(x) \in \ci$ the maximal $\tx \in \ci$ such that $\tx \leq x$. Likewise, for $\bx \in [0,1]^d$ we denote $\trunc(\bx) = (\trunc(x_1),\ldots,\trunc(x_d)) \in \ci^d$.

Since $f$ is benign, there is an exponential-time algorithm $\ca$, such that given $\hbx \in \{0,1\}^{c \cdot d}$ it returns $\ca(\hbx) \in \{0,1\}^c$, such that 
\begin{equation}
\label{eq:def hatf - depth}
	|f(\real(\hbx))-\real(\ca(\hbx))| \leq \frac{1}{2^c} \leq \frac{1}{p(d)}~.
\end{equation}
Let $\hat{f}:\{0,1\}^{c \cdot d} \rightarrow \{0,1\}^c$ be the function that this algorithm computes. That is, $\hat{f}(\hbx)=\ca(\hbx)$.
Assume 
%toward contradiction 
that the function $\hat{f}$ can be computed by a threshold circuit $T$ of size $\poly(d)$ and depth $k-2$. We will construct a neural network $N$ of size $\poly(d)$ and depth $k$ such that $\snorm{f-N}_{L_2(\mu)} \leq \epsilon$ and thus reach a contradiction. It implies that $\hat{f}$ can be computed in exponential time but cannot be computed by a threshold circuit of size $\poly(d)$ and depth $k-2$.
Then, the theorem follows from the following lemma (see proof in Section~\ref{sec:proof of lemma to output dim 1}).
\begin{lemma}
\label{lemma:to output dim 1}
	Let $l(d) = \co(\log(d))$ be monotonically non-decreasing and let $g:\{0,1\}^{d \cdot l} \rightarrow \{0,1\}^l$ be a function that can be computed in exponential time, and cannot be computed by a threshold circuit of size $\poly(d)$ and constant depth $m$. Then, there is a function $g':\{0,1\}^{d'} \rightarrow \{0,1\}$ that can be computed in exponential time, and cannot be computed by a threshold circuit of size $\poly(d')$ and depth $m$, i.e., $g'  \in \EXP \setminus \TC^0_m$.
\end{lemma}

Let $\tilde{f}:[0,1]^d \rightarrow \ci^d$ be such that $\tilde{f}(\bx) = \real(\hat{f}(\bin(\trunc(\bx))))$. Thus, $\tilde{f}$ transforms $\bx$ to a $(c \cdot d)$-bits binary representation, runs $\hat{f}$, and converts the output from binary to a real value.
Let $\bx \in [0,1]^d$, let $\tbx = \trunc(\bx)$ and let $\hbx = \bin(\tbx)$. 
By Eq.~\ref{eq:def hatf - depth} we have 
\[
|\tilde{f}(\bx)-f(\tbx)|  
= |\real(\hat{f}(\hbx)) - f(\real(\hbx))|
\leq \frac{1}{p(d)}
= \frac{\epsilon}{4 L \sqrt{d}} \leq \frac{\epsilon}{4}~.
\] 
Also, since $f$ is $L$-Lipschitz then we have 
\[
|f(\tbx)-f(\bx)| 
\leq L \cdot \norm{\tbx-\bx} 
\leq L \cdot \frac{\sqrt{d}}{2^c}
\leq L \cdot \frac{\sqrt{d}}{p(d)}
=\frac{L \sqrt{d} \cdot \epsilon}{4 L \sqrt{d}}
= \frac{\epsilon}{4}~.
\]
Thus, $|\tilde{f}(\bx)-f(\bx)| \leq |\tilde{f}(\bx)-f(\tbx)| + |f(\tbx)-f(\bx)|  \leq \frac{\epsilon}{2}$, and therefore $\snorm{f-\tilde{f}}_{L_2(\mu)} \leq \frac{\epsilon}{2}$.

We now construct a network $N$ of size $\poly(d)$ and depth $k$ such that $\snorm{\tilde{f}-N}_{L_2(\mu)} \leq \frac{\epsilon}{2}$. It implies that $\snorm{f-N}_{L_2(\mu)} \leq \snorm{f-\tilde{f}}_{L_2(\mu)} + \snorm{\tilde{f}-N}_{L_2(\mu)} \leq \epsilon$ and thus completes the proof.
The network $N$ consists of three parts. First, it transforms the input $\bx \in [0,1]^d$ w.h.p. to $\bin(\trunc(\bx)) \in  \{0,1\}^{c \cdot d}$. Then, it simulates the threshold circuit $T$. Finally, it converts the output of $T$ from a binary representation to the corresponding real value. We now describe these parts in more details.

For the transformation from $\bx \in [0,1]^d$ to $\bin(\trunc(\bx)) \in  \{0,1\}^{c \cdot d}$ we will need the following lemma (see proof in Section~\ref{sec:proof of lemma to binary}).
\begin{lemma}
	\label{lemma:to binary}
	Let $\delta = \frac{1}{\poly(d)}$.
	There is a neural network $\cn$ of depth $2$, size $\poly(d)$, and $(c \cdot d)$ outputs, such that 
	\[
	\Pr_{\bx \sim \mu}\left[ \cn(\bx) = \bin(\trunc(\bx)) \right] \geq 1 - \delta~.
	\] 
\end{lemma}

Also, for the simulation of the threshold circuit $T$ we will need the following lemma (see proof in Section~\ref{sec:proof of lemma from TC to NN}).
\begin{lemma}
	\label{lemma:from TC to NN}
	Let $T$ be a threshold circuit with $d$ inputs, $q$ outputs, depth $m$ and size 
	%$\poly(d)$. 
	$s$.
	There is a neural network $\cn$ with $q$ outputs, depth $m+1$ and size 
	%$\poly(d)$, 
	$2s+q$,
	such that for every $\bx \in \{0,1\}^d$ we have $\cn(\bx) = T(\bx)$.	
	Moreover, for every input $\bx \in \reals^d$ the outputs of $\cn$ are in $[0,1]$.
\end{lemma}

We note that lemmas with a similar idea to Lemmas~\ref{lemma:to binary} and~\ref{lemma:from TC to NN} where also shown in \cite{vardi2020neural}.
The construction of $N$ proceeds as follows. 
Let $\delta= \frac{\epsilon^2}{4}$.
First $N$ transforms w.p. at least $1-\delta$ the input $\bx \in [0,1]^d$ to $\hbx = \bin(\trunc(\bx)) \in  \{0,1\}^{c \cdot d}$. By Lemma~\ref{lemma:to binary} it can be done by a depth-$2$ network $\cn_1$. Second, $N$ computes $T(\hbx)$. By Lemma~\ref{lemma:from TC to NN} it can be done by a network $\cn_2$ of depth $k-1$. 
Note that 
\begin{align*}
	T(\hbx) 
	= \hat{f}(\hbx) 
	= \bin(\real(\hat{f}(\hbx) ))
	= \bin(\tilde{f}(\bx))~.
\end{align*}
Third, $N$ transforms the output of $\cn_2$ to the corresponding value in $\ci$, and thus obtains $\tilde{f}(\bx)$. It can be done by a single layer, since if $\hat{\bz} \in \{0,1\}^{c}$ is a binary representation of $z \in \ci$, then 
\begin{equation}
	\label{eq:final sum}
	z = \sum_{j \in [c]}\hat{z}_j \cdot \frac{2^{j-1}}{2^c}~.
\end{equation}
Since the final layers in $\cn_1$ and $\cn_2$ do not have activations and can be combined with the next layers, and since the third part of $N$ is simply a linear transformation, then the depth of $N$ is $k$.

Given an input $\bx \sim \mu$, the network $N$ computes $\tilde{f}(\bx)$ w.p. at least $1-\delta$. However, it is possible (w.p. at most $\delta$) that $\cn_1$ fails to transform the input $\bx$ to $\bin(\trunc(\bx))$, and therefore $N$ fails to compute $\tilde{f}(\bx)$. Still, even in this case we can bound the output of $N$ as follows. If $\cn_1$ fails to transform the input $\bx$ to $\bin(\trunc(\bx))$, then the input to $\cn_2$ may contain values other than $\{0,1\}$. However, by Lemma~\ref{lemma:from TC to NN}, the network $\cn_2$ outputs only values in $[0,1]$. Hence, when computing the output of $N$ using Eq.~\ref{eq:final sum}, the resulting value is at least $0$ and at most $(2^c - 1) \cdot \frac{1}{2^c} \leq 1$.
Therefore, we have $N(\bx) \in [0,1]$. Since $\tilde{f}(\bx) \in [0,1]$, then $|\tilde{f}(\bx) - N(\bx)| \leq 1$. 
We have
\[
\E_{\bx \sim \mu}\left(\tilde{f}(\bx) - N(\bx) \right)^2 \leq \delta \cdot 1^2 + (1-\delta) \cdot 0 = \delta = \frac{\epsilon^2}{4}~.
\]
Hence, $\snorm{\tilde{f}-N}_{L_2(\mu)} \leq \frac{\epsilon}{2}$ as required.

\subsubsection{Proof of Lemma~\ref{lemma:to output dim 1}}
%\subsection{Proof of Lemma~\ref{lemma:to output dim 1}}
\label{sec:proof of lemma to output dim 1}

Let $g':\{0,1\}^{d'} \rightarrow \{0,1\}$ be a function such that if $d' = d \cdot l + l$ then we have the following. Let $\bx \in \{0,1\}^{d'}$ and denote $\bx^1 = (x_1,\ldots,x_{d \cdot l})$ and $\bx^2 = (x_{d \cdot l+1},\ldots x_{d \cdot l+l})$. If $\bx^2$ has a $1$-bit in the $i$-th coordinate and all other bits are $0$, then we say that $\bx^2$ is the {\em $i$-selector}. For $\bx \in \{0,1\}^{d'}$ such that $\bx^2$ is $i$-selector, we have $g'(\bx)=(g(\bx^1))_i$. Namely, $g'$ returns the $i$-th output bit of $g(\bx^1)$.
Since $g$ can be computed in exponential time then clearly $g'$ can also be computed in exponential time. 
Assume that $g'$ can be computed by a threshold circuit $T'$ of size $s(d')=\poly(d')$ and depth $m$. Then, $g$ can also be computed by a $\poly(d)$-sized threshold circuit $T$ of depth $m$ as follows. The circuit $T$ consists of $l$ circuits $T_1,\ldots,T_l$, such that $T_i$ computes the $i$-th output bit. The circuit $T_i$ has input dimension $d \cdot l$, and is obtained from $T'$ by hardwiring the input bits $\bx^2$ to be the $i$-selector. That is, let $n$ be a threshold gate in the first layer of $T'$, and assume that the weight from the $i$-th component of $\bx^2$ to $n$ is $w$, and that the bias of $n$ is $b$. Then, in $T_i$ we change the bias of $n$ to $b+w$. Note that $T$ has size $l \cdot s(d l + l) = \poly(d)$ and depth $m$, and that $T$ computes $g$.

\subsubsection{Proof of Lemma~\ref{lemma:to binary}}
%\subsection{Proof of Lemma~\ref{lemma:to binary}}
\label{sec:proof of lemma to binary}

Let $\bx \in [0,1]^d$.
In order to construct $\cn$, we need to show how to compute $\bin(\trunc(x_i))$ for every $i \in [d]$.
We will show a depth-$2$ network $\cn'$ such that given $x_i \sim \mu_i$ it outputs $\bin(\trunc(x_i))$ w.p. $\geq 1 - \frac{\delta}{d}$. Then, the network $\cn$ consists of $d$ copies of $\cn'$, and satisfies
\[
\Pr_{\bx \sim \mu}\left[\cn(\bx) \neq \bin(\trunc(\bx))\right]
\leq \sum_{i \in [d]} \Pr_{x_i \sim \mu_i}\left[\cn'(x_i) \neq \bin(\trunc(x_i))\right]
\leq \frac{\delta}{d} \cdot d = \delta~.
\]

We denote $\tx_i = \trunc(x_i)$.
For $j \in [c]$ let $I_j \subseteq \{0, \ldots, 2^c - 1\}$ be the integers such that the $j$-th bit in their binary representation is $1$.
Hence, given $x_i$, the network $\cn'$ should output in the $j$-th output $\onefunc_{I_j}(2^c \cdot \tx_i)$,
where $\onefunc_{I_j}(z)=1$ if $z \in I_j$ and $\onefunc_{I_j}(z)=0$ otherwise.

Since $\mu$ has a polynomially-bounded marginal density, then there is $\Delta = \frac{1}{\poly(d)}$ such that for every $i \in [d]$ and every $t \in [0,1]$ we have
\begin{equation}
\label{eq:bounded marginal}
	\Pr_{\bx \sim \mu}\left[x_i \in \left[t-\frac{\Delta}{2^c},t\right] \right] \leq \frac{\delta}{2^c \cdot d}~.
\end{equation}

For an integer $0 \leq l \leq 2^c - 1$, let $g_l:\reals \rightarrow \reals$ be such that
\[
g_l(t) = \left[\frac{1}{\Delta}\left(t-l+\Delta\right)\right]_+ - \left[\frac{1}{\Delta}\left(t-l\right)\right]_+~.
\]
Note that $g_l(t)=0$ if $t \leq l - \Delta$, and that $g_l(t)=1$ if $t \geq l$.
Let $g'_l(t) = g_{l}(t)-g_{l+1}(t)$.
Note that $g'_l(t)=0$ if $t \leq l - \Delta$ or $t \geq l + 1$, and that $g'_l(t)=1$ if $l \leq t \leq l+1-\Delta$.

Let $h_j(t) = \sum_{l \in I_j}g'_l(t)$.
Note that for every $l \in \{0, \ldots, 2^c - 1\}$ and $l \leq t \leq l+1-\Delta$ we have $h_j(t)=1$ if $l \in I_j$ and $h_j(t)=0$ otherwise.
Therefore, if $h_j(2^c x_i) \neq \onefunc_{I_j}(2^c \tx_i)$ then $2^c x_i \in [l + 1-\Delta,l+1]$ for some integer $0 \leq l \leq 2^c-1$.

Let $\cn'$ be such that $\cn'(x_i)=\left(h_1(2^c x_i),\ldots,h_{c}(2^c x_i)\right)$.
Note that $\cn'$ can be implemented by a depth-$2$ neural network.
We have:
\begin{align*}
	\Pr_{x_i \sim \mu_i}\left[\cn'(x_i) \neq \bin(\tx_i) \right]
	&= \Pr_{x_i \sim \mu_i}\left(\exists j \in [c] \text{\ s.t.\ } h_j(2^c x_i) \neq (\bin(\tx_i))_j \right)
	\\
	&=\Pr_{x_i \sim \mu_i}\left[\exists j \in [c] \text{\ s.t.\ } h_j(2^c x_i) \neq \onefunc_{I_j}(2^c \tx_i) \right]
	\\
	&\leq \Pr_{x_i \sim \mu_i}\left[ 2^c x_i \in[l + 1-\Delta,l+1], 0 \leq l \leq 2^c - 1 \right]
	\\
	&\leq \sum_{0 \leq l \leq 2^c - 1} \Pr_{x_i \sim \mu_i}\left[x_i \in \left[\frac{l}{2^c}+\frac{1}{2^c}-\frac{\Delta}{2^c}, \frac{l}{2^c}+\frac{1}{2^c}\right] \right]
	\\
	&\stackrel{(Eq.~\ref{eq:bounded marginal})}{\leq} 2^c \cdot \frac{\delta}{2^c \cdot d}
	= \frac{\delta}{d}~.
\end{align*}

\subsubsection{Proof of Lemma~\ref{lemma:from TC to NN}}
%\subsection{Proof of Lemma~\ref{lemma:from TC to NN}}
\label{sec:proof of lemma from TC to NN}

Let $g$ be a gate in $T$, and let $\bw \in \integers^l$ and $b \in \integers$ be its weights and bias. Let $n_1$ be a neuron with weights $\bw$ and bias $b$, and let $n_2$ be a neuron with weights $\bw$ and bias $b-1$. Let $\by \in \{0,1\}^l$. Since $(\inner{\bw,\by}+b) \in \integers$, we have $[\inner{\bw,\by}+b]_+ - [\inner{\bw,\by}+b-1]_+ = \sign(\inner{\bw,\by} + b)$. Hence, the gate $g$ can be replaced by the neurons $n_1,n_2$. We replace all gates in $T$ by neurons and obtain a network $\cn$. Since each output gate of $T$ is also replaced by two neurons, $\cn$ has $m+1$ layers 
and size $2s+q$.
Since for every $\bx \in \reals^d$, weight vector $\bw$ and bias $b$ we have $[\inner{\bw,\bx}+b]_+ - [\inner{\bw,\bx}+b-1]_+ \in [0,1]$ then for every input $\bx \in \reals^d$ the outputs of $\cn$ are in $[0,1]$.

\subsection{Proof of Theorem~\ref{thm:counting}}
%\section{Proof of Theorem~\ref{thm:counting}}
\label{app:proof of theorem counting}

Every neural network can be represented in a standard way by a binary vector, such that if the network has $\poly(d)$ neurons and the binary representation of each weight is of length at most $\poly(d)$, then the binary representation of the network is of length $\poly(d)$. In the following lemma we show that for a sufficiently large $d$, even the set of networks whose binary representations are of length $d^{\log(d)}$ does not suffice to approximate all $1$-Lipschits functions.

\begin{lemma}
\label{lemma:1 lipschits}
	Let $\epsilon = \frac{1}{d}$. There is a distribution $\mu$ with a polynomially-bounded marginal density, such that for every sufficiently large $d$ we have the following: There is a $1$-Lipscitz function $g:[0,1]^d \rightarrow [0,1]$ such that for every neural network $N$ whose binary representation has $d^{\log(d)}$ bits, we have $\snorm{g-N}_{L_2(\mu)} > \epsilon$.
\end{lemma}
\begin{proof}
	For $\bz \in \{0,1\}^d$ we denote $A_\bz = \{\bx \in [0,1]^d: \forall i \in [d], \; |x_i-z_i| \leq \frac{1}{4}\}$. Thus, $A_\bz$ is a cube with volume $\left(\frac{1}{4}\right)^d$. For $\bx \in [0,1]^d$ we denote $\text{dist}(\bx,A_\bz)=\min\{\norm{\bx-\ba}: \ba \in A_\bz\}$. For $\bz \in \{0,1\}^d$, let $h_\bz:[0,1]^d \rightarrow [0,1]$ be such that $h_\bz(\bx) = \max\{0, \frac{1}{4} - \text{dist}(\bx,A_\bz)\}$. Note that if $\bx \in A_\bz$ then $h_\bz(\bx)=\frac{1}{4}$, if there is $i \in [d]$ such that $|x_i-z_i| \geq \frac{1}{2}$ then $h_\bz(\bx)=0$, and $h_\bz$ is $1$-Lipschitz. For $\psi:\{0,1\}^d \rightarrow \{0,1\}$ let $f_\psi:[0,1]^d \rightarrow [0,1]$ be such that 
	$f_\psi(\bx) = \sum_{\bz \in \{0,1\}^d}\psi(\bz) h_\bz(\bx)$.
	Note that for every $\bz \in \{0,1\}^d$ and $\bx \in A_\bz$ we have $f_\psi(\bx) = \frac{1}{4} \psi(\bz)$. Moreover, since for every $\bx \in [0,1]$ there is at most one $\bz \in \{0,1\}^d$ such that $h_\bz(\bx) \neq 0$, then $f_\psi$ is also $1$-Lipschitz. Let $\cf = \{f_\psi: \psi \in \{0,1\}^{(\{0,1\}^d)}\}$. 
	Let $\mu$ be the uniform distribution over $([0,1/4] \cup [3/4,1])^d = \bigcup_{\bz \in \{0,1\}^d}A_\bz$. Note that $\mu$ has a polynomially-bounded marginal density.
	
	Let $\cg$ be the set of all $1$-Lipschitz functions $g:[0,1]^d \rightarrow [0,1]$. Note that $\cf \subseteq \cg$.
	Let $\ch$ be a set of functions, such that for every $g \in \cg$ there exists a function $h \in \ch$ such that $\snorm{g-h}_{L_2(\mu)} \leq \epsilon$. We now show a lower bound on the size of $\ch$. Then, we will use this bound to show that the set of networks whose binary representations are of length $d^{\log(d)}$ does not suffice to approximate $\cg$.
	
	Since $\cf \subseteq \cg$, then for every $f \in \cf$ there exists a function $h \in \ch$ such that $\snorm{f-h}_{L_2(\mu)} \leq \epsilon$. 
	For a function $\varphi \in \ch \cup \cf$ and $\delta>0$, we denote $B_\delta(\varphi) = \{f \in \cf: \snorm{f-\varphi}_{L_2(\mu)} \leq \delta\}$.
	Let $h \in \ch$. We will first bound the size of $B_\epsilon(h)$ and then use it to obtain a lower bound for $|\ch|$.
	For every $f_1,f_2 \in B_\epsilon(h)$, we have $\norm{f_1-f_2}_{L_2(\mu)} \leq 2 \epsilon$. Hence, for every $f' \in B_\epsilon(h)$, we have $B_\epsilon(h) \subseteq B_{2\epsilon}(f')$.
	Let $\psi' \in \{0,1\}^{(\{0,1\}^d)}$ be such that $f'=f_{\psi'}$. Let $\psi \in  \{0,1\}^{(\{0,1\}^d)}$ be such that $f_\psi \in B_{2\epsilon}(f')$. We have
	\begin{align*}
		(2 \epsilon)^2 
		&\geq \norm{f_{\psi'}-f_\psi}^2_{L_2(\mu)}
		= \int_{\bx \in [0,1]^d} \left( f_{\psi'}(\bx) - f_\psi(\bx) \right)^2 \mu(\bx) d\bx
		\\
		&= \sum_{\bz \in \{0,1\}^d} \int_{\bx \in A_\bz} \left( f_{\psi'}(\bx) - f_\psi(\bx) \right)^2 \cdot \frac{1}{2^d \cdot \left(\frac{1}{4}\right)^d} d\bx
		\\
		&= 2^d \cdot \sum_{\bz \in \{0,1\}^d} \int_{\bx \in A_\bz} \frac{1}{16} \left( \psi'(\bz) - \psi(\bz) \right)^2 d\bx
		\\
		&= \frac{2^d}{16} \cdot \sum_{\bz \in \{0,1\}^d} \left(\frac{1}{4}\right)^d \onefunc(\psi'(\bz) \neq \psi(\bz))~.
	\end{align*}
	Therefore, $(2 \epsilon)^2 \cdot 16 \cdot 2^d \geq \sum_{\bz \in \{0,1\}^d} \onefunc(\psi'(\bz) \neq \psi(\bz))$.
	Thus, for every $f_\psi \in B_{2\epsilon}(f')$, the function $\psi$ disagrees with $\psi'$ in at most $(2 \epsilon)^2  \cdot 16 \cdot 2^d$ points. Hence, we have
	\[
		|B_\epsilon(h)|
		\leq |B_{2\epsilon}(f')| 
		\leq \sum_{j=0}^{(2 \epsilon)^2 \cdot 16 \cdot 2^d} \binom{2^d}{j}
		\leq (2^d+1) ^{(2 \epsilon)^2 \cdot 16 \cdot 2^d}
		\leq  2^{(d+1) \cdot (2 \epsilon)^2 \cdot 16 \cdot 2^d}~.
	\]
	Since $\cf = \bigcup_{h \in \ch} B_\epsilon(h)$, then 
	\[
		|\ch| 
		\geq \frac{|\cf|}{|B_\epsilon(h)|} 
		\geq \frac{2^{2^d}}{2^{(d+1) \cdot (2 \epsilon)^2  \cdot 16 \cdot 2^d}}
		\geq \frac{2^{2^d}}{2^{8  \cdot 16 d \epsilon^2 \cdot 2^d}}
		= 2^{  \left( 1 -  8  \cdot 16 d \epsilon^2 \right)2^d }~.
	\]
	By plugging-in $\epsilon=\frac{1}{d}$, we have for a suffciently large $d$ that 
	\begin{equation}
	\label{eq:bound ch}
		|\ch| 
		\geq 2^{  \left( 1 -  \frac{8  \cdot 16}{d} \ \right)2^d }
		\geq 2^{ 2^{d-1} }~.
	\end{equation}
	
	Let $\cn$ be the set of all functions that can be expressed by a neural network whose representation has $d^{\log(d)}$ bits. By Eq.~\ref{eq:bound ch}, if for every $1$-Lipschitz function $g \in \cg$ there exists a network $N \in \cn$ such that $\snorm{g-N}_{L_2(\mu)} \leq \epsilon$, then $|\cn| \geq 2^{ 2^{d-1} }$. 
	However, for a sufficiently large $d$ we have 
	\[
		|\cn| = 2^{d^{\log(d)}} = 2^{2^{\log^2(d)}} < 2^{ 2^{d-1} }~.
	\] 
	Therefore, for every sufficiently large $d$, there is a $1$-Lipschitz function $g:[0,1]^d \rightarrow [0,1]$ such that for every network $N$ whose binary representation has $d^{\log(d)}$ bits we have $\snorm{g-N}_{L_2(\mu)} > \epsilon$.
\end{proof}

Finally, in order to prove the theorem we need to obtain a sequence of functions $\{f_d\}_{d=1}^\infty$ where $f_d:[0,1]^d \rightarrow [0,1]$, and some $\epsilon=\frac{1}{\poly(d)}$ and distribution $\mu$ with a polynomially-bounded marginal density. Then, we need to show that for every polynomial $p(d)$ and every sequence of neural networks $\{N_d\}_{d=1}^\infty$, where the input dimension of $N_d$ is $d$ and its size and number of bits in the weights are bounded by $p(d)$, we have for some $d$ that $\snorm{f_d - N_d}_{L_2(\mu)} > \epsilon$. By Lemma~\ref{lemma:1 lipschits}, for every sufficiently large $d$ there is a $1$-Lipschitz function $g_d:[0,1]^d \rightarrow [0,1]$ such that for every network $N_d$ whose binary representation has $d^{\log(d)}$ bits we have $\snorm{g_d-N_d}_{L_2(\mu)} > \epsilon$. Consider a sequence $\{f_d\}_{d=1}^\infty$ where for every sufficiently large $d$ we choose $f_d=g_d$. Let $\{N_d\}_{d=1}^\infty$ be a sequence of networks such that their sizes and number of bits in the weights are bounded by some $\poly(d)$. The length of the binary representation of $N_d$ is bounded by some $\poly(d)$, and hence for a sufficiently large $d$ it is smaller than $d^{\log(d)}$. Hence, for a sufficiently large $d$ we have $\snorm{f_d-N_d}_{L_2(\mu)} > \epsilon$.

\section{Proofs for Section~\ref{sec:barriers size}}

\subsection{Proof of Theorem~\ref{thm:bounded size}}
%\section{Proof of Theorem~\ref{thm:bounded size}}
\label{app:proof of theorem bounded size}

Let $f:[0,1]^d \rightarrow [0,1]$ be a semi-benign function. Assume that $f$ is $L$-Lipschitz for $L=2^{\poly(d)}$. Let $\epsilon=\frac{1}{\poly(d)}$, let 
$q(d)=\frac{4L\sqrt{d}}{\epsilon}$, 
and let $c=\ceil{\log(q(d)+1)}$. Let $\ci = \{\frac{j}{2^c}: 0 \leq j \leq  2^c - 1, j \in \integers\}$. 
Note that since $L$ is exponential, then $q(d)$ is also exponential, and that $c=\poly(d)$.
We use the notations $\bin(\cdot),\real(\cdot)$ and $\trunc(\cdot)$ in an analogous way to the proof of Theorem~\ref{thm:bounded depth}.

Let $c' = \log(4/\epsilon)$. Since $f$ is semi-benign, there is an exponential-time algorithm $\ca$, such that given $\hbx \in \{0,1\}^{c \cdot d}$ it computes $f(\real(\hbx))$ within precision of $c'$ bits. In order to simplify notations, we assume that $\ca(\hbx) \in \{0,1\}^c$. Namely, $\ca$ computes a $c'$-bits approximation of $f(\real(\hbx))$ and then pads it with zeros to obtain the $c$-bits output. Thus, we have
\begin{equation}
	\label{eq:def hatf - size}
	|\real(\ca(\hbx)) - f(\real(\hbx))| \leq \frac{1}{2^{c'}} = \frac{\epsilon}{4}~.
\end{equation}
Let $\hat{f}:\{0,1\}^{c \cdot d} \rightarrow \{0,1\}^c$ be the function that this algorithm computes. That is, $\hat{f}(\hbx)=\ca(\hbx)$.
Assume 
%toward contradiction 
that the function $\hat{f}$ can be computed by a threshold circuit $T$ of size $\poly(d)$. We will construct a neural network $N$ of size $\poly(d)$ such that $\snorm{f-N}_{L_2(\mu)} \leq \epsilon$ and thus reach a contradiction. It implies that $\hat{f}$ can be computed in exponential time but cannot be computed by a $\poly(d)$-sized threshold circuit.
Then, the theorem follows from the following lemma (whose proof is similar to the proof of Lemma~\ref{lemma:to output dim 1}).
\begin{lemma}
	\label{lemma:to output dim 1 - size}
	Let $l(d) \leq \poly(d)$ be monotonically non-decreasing and let $g:\{0,1\}^{d \cdot l} \rightarrow \{0,1\}^l$ be a function that can be computed in exponential time, and cannot be computed by a threshold circuit of size $\poly(d)$. Then, there is a function $g':\{0,1\}^{d'} \rightarrow \{0,1\}$ that can be computed in exponential time, and cannot be computed by a threshold circuit of size $\poly(d')$, i.e., $g' \in \EXP \setminus \Ppoly$.
\end{lemma}

Let $\tilde{f}:[0,1]^d \rightarrow \ci^d$ be such that $\tilde{f}(\bx) = \real(\hat{f}(\bin(\trunc(\bx))))$. Thus, $\tilde{f}$ transforms $\bx$ to a $(c \cdot d)$-bits binary representation, computes $\hat{f}$, and converts the output from binary to a real value.
Let $\bx \in [0,1]^d$, let $\tbx = \trunc(\bx)$ and let $\hbx = \bin(\tbx)$. 
By Eq.~\ref{eq:def hatf - size} we have 
\[
|\tilde{f}(\bx)-f(\tbx)|  
= |\real(\hat{f}(\hbx)) - f(\real(\hbx))|
\leq \frac{\epsilon}{4}~.
\] 
Also, since $f$ is $L$-Lipschitz then we have 
\[
|f(\tbx)-f(\bx)| 
\leq L \cdot \norm{\tbx-\bx} 
\leq L \cdot \frac{\sqrt{d}}{2^c}
\leq L \cdot \frac{\sqrt{d}}{q(d)}
=\frac{L \sqrt{d} \cdot \epsilon}{4 L \sqrt{d}}
= \frac{\epsilon}{4}~.
\]
Thus, $|\tilde{f}(\bx)-f(\bx)| \leq |\tilde{f}(\bx)-f(\tbx)| + |f(\tbx)-f(\bx)|  \leq \frac{\epsilon}{2}$, and therefore $\snorm{f-\tilde{f}}_{L_2(\mu)} \leq \frac{\epsilon}{2}$.

We now construct a $\poly(d)$-sized network $N$ such that $\snorm{\tilde{f}-N}_{L_2(\mu)} \leq \frac{\epsilon}{2}$. It implies that $\snorm{f-N}_{L_2(\mu)} \leq \snorm{f-\tilde{f}}_{L_2(\mu)} + \snorm{\tilde{f}-N}_{L_2(\mu)} \leq \epsilon$ and thus completes the proof.

The construction of $N$ follows a similar idea to the proof of Theorem~\ref{thm:bounded depth}:
First, $N$ transforms w.p. at least $1-\frac{\epsilon^2}{4}$ the input $\bx \in [0,1]^d$ to $\hbx = \bin(\trunc(\bx)) \in \{0,1\}^{c \cdot d}$. Then, it computes $T(\hbx)$ by simulating the threshold circuit $T$ using Lemma~\ref{lemma:from TC to NN}. Finally, it transforms $T(\hbx)$ from a binary representation to the corresponding real value.  
The main difference from the proof of Theorem~\ref{thm:bounded depth}, is that since $L$ is exponential, then the number of bits in $\hbx$ is polynomial in $d$ (rather than logarithmic), and hence computing this binary representation with a $\poly(d)$-sized network requires a more clever construction.
In the following lemma, we show that the transformation from $\bx \in [0,1]^d$ to $\bin(\trunc(\bx)) \in \{0,1\}^{c \cdot d}$ can be implemented, using the construction of \cite{telgarsky2016benefits}, with a $\poly(d)$-sized network. 
Then, the construction of $N$ and the proof that it approximates $\tilde{f}$ follow similar arguments to the proof of Theorem~\ref{thm:bounded depth}.

\begin{lemma}
	\label{lemma:to binary exponential}
	Let $\delta = \frac{1}{\poly(d)}$.
	There is a neural network $\cn$ of size $\poly(d)$ and $(c \cdot d)$ outputs, such that 
	\[
	\Pr_{\bx \sim \mu}\left[ \cn(\bx) = \bin(\trunc(\bx)) \right] \geq 1 - \delta~.
	\] 
\end{lemma}
\begin{proof}
	Let $\bx \in [0,1]^d$. In order to construct $\cn$, we need to show how to compute $\bin(\trunc(x_i))$ for every $i \in [d]$. We construct a network $\cn'$ such that for every $i \in [d]$, given $x_i \sim \mu_i$ it outputs $\bin(\trunc(x_i))$ w.p. at least $1-\frac{\delta}{d}$.
	Then, the network $\cn$ consists of $d$ copies of $\cn'$, and satisfies
	\[
	\Pr_{\bx \sim \mu}\left[\cn(\bx) \neq \bin(\trunc(\bx))\right]
	\leq \sum_{i \in [d]} \Pr_{x_i \sim \mu_i}\left[\cn'(x_i) \neq \bin(\trunc(x_i))\right]
	\leq \frac{\delta}{d} \cdot d = \delta~.
	\]
	
	The network $\cn'$ consists of $c$ networks $\cn_1,\ldots,\cn_c$, such that for every $i \in [d]$ the network $\cn_j$ computes the $j$-th bit of $\bin(\trunc(x_i))$ w.p. at least $1-\frac{\delta}{d \cdot c}$ over $\mu_i$. Hence, the network $\cn'$ satisfies 
	\[
	\Pr_{x_i \sim \mu_i}\left[\cn'(x_i) \neq \bin(\trunc(x_i))\right]
	\leq \sum_{j \in [c]} \Pr_{x_i \sim \mu_i}\left[\cn_j(x_i) \text{ fails}\right]
	\leq \frac{\delta}{d \cdot c} \cdot c = \frac{\delta}{d}~.
	\]
	
	We now construct $\cn_j$. Note that in order to compute the $j$-th bit of $\bin(\trunc(x_i))$, the network $\cn_j$ needs to oscillate $\co(2^j)$ many times. 
	%Thus, the number of linear pieces in $\cn_j$ may be exponential in $d$. 
	%Since the number of linear pieces in a ReLU network is polynomial in the width and exponential in the depth \citep{telgarsky2015representation}, then the network $\cn_j$ cannot have a constant depth. 
	Hence, unlike the depth-$2$ construction from Lemma~\ref{lemma:to binary}, the network $\cn_j$ requires $\poly(d)$ depth.
	We note that a similar construction was used in \cite{safran2017depth}.
	
	In the following, we assume that $x_i \not \in \ci \cup \{1\}$, i.e., $x_i \cdot 2^c$ is not an integer. Note that since $\mu$ has a polynomially-bounded marginal density then the probability that $x_i \in \ci \cup \{1\}$ is $0$, and hence we can ignore this case.
	Let $\varphi(z) = [2z]_+ - [4z-2]_+$. \cite{telgarsky2016benefits} observed that the composition of $\varphi$ with itself $j$ times, denoted by $\varphi^j$, yields a highly oscillatory triangle wave function. In the domain $[0,1]$, the function $\varphi^j$ consists of $2^{j-1}$ identical triangles of height $1$. 
	Note that for $z \leq 0$ we have $\varphi^j(z)=0$.
	Given $z \in (0,1) \setminus \ci$, note that the $j$-th bit of $\bin(\trunc(z))$ is $1$ iff the following expression is at least $\frac{1}{2}$:
	\[
	\varphi^j \left(z - \frac{1}{2} \cdot \frac{1}{2^j}\right).
	\]
	Hence, given $x_i \sim \mu_i$, the network $\cn_j$ should return w.p. at least $1-\frac{\delta}{d \cdot c}$ the expression 
	\[
	\onefunc_{\geq \frac{1}{2}}\left(\varphi^j \left(x_i - 2^{-j-1}\right)\right),
	\]
	where $\onefunc_{\geq \frac{1}{2}}(y)=1$ if $y \geq \frac{1}{2}$ and is $0$ otherwise. While the function $\onefunc_{\geq \frac{1}{2}}$ cannot be expressed by a ReLU network, it can be approximated by
	\[
	h_\Delta(y)
	= \left[\frac{1}{\Delta}\left(y-\frac{1}{2}+\frac{\Delta}{2}\right)\right]_+ - \left[\frac{1}{\Delta}\left(y-\frac{1}{2}-\frac{\Delta}{2}\right)\right]_+~.
	\]
	Note that $h_\Delta(y)=0$ for every $y \leq \frac{1}{2}-\frac{\Delta}{2}$, and $h_\Delta(y)=1$ for every $y \geq \frac{1}{2}+\frac{\Delta}{2}$. Since $\mu$ has a polynomially-bounded marginal density, then by choosing a sufficiently small $\Delta=\frac{1}{\poly(d)}$, we have w.p. at least $1-\frac{\delta}{d \cdot c}$ over $x_i \sim \mu_i$, that 
	\[
	h_\Delta\left(\varphi^j \left(x_i - 2^{-j-1}\right)\right) =
	\onefunc_{\geq \frac{1}{2}}\left(\varphi^j \left(x_i -2^{-j-1} \right)\right)~.
	\]
	Finally, the l.h.s. of the above equation can be implemented by a $\poly(d)$-sized neural network $\cn_j$. The construction of such a network is straightforward, since it is a composition of a $\poly(d)$ number of functions, that can be implemented by ReLU networks of size $\poly(d)$.
\end{proof}

\subsection{Proof of Theorem~\ref{thm:bounded size2}}
%\section{Proof of Theorem~\ref{thm:bounded size2}}
\label{app:proof of theorem bounded size2}

The proof uses ideas from the proofs of Theorems~\ref{thm:bounded depth} and~\ref{thm:bounded size} with some necessary modifications.
Let $f:[0,1]^d \rightarrow [0,1]$ be a polynomial-time benign function. Assume that $f$ is $L$-Lipschitz for $L=\poly(d)$. Let $\epsilon=\frac{1}{\poly(d)}$, let 
$p(d)=\frac{4L\sqrt{d}}{\epsilon}$, 
and let $c=\ceil{\log(p(d)+1)}$. Let $\ci = \{\frac{j}{2^c}: 0 \leq j \leq  2^c - 1, j \in \integers\}$. 
We use the notations $\bin(\cdot),\real(\cdot)$ and $\trunc(\cdot)$ in a similar way to the proof of Theorem~\ref{thm:bounded depth}.

Since $f$ is polynomial-time benign, there is a polynomial-time algorithm $\ca$, such that given $\hbx \in \{0,1\}^{c \cdot d}$ it returns $\ca(\hbx) \in \{0,1\}^c$, such that 
\begin{equation*}
	\label{eq:def hatf - size2}
	|f(\real(\hbx))-\real(\ca(\hbx))| \leq \frac{1}{2^c} \leq \frac{1}{p(d)}~.
\end{equation*}
Let $\hat{f}:\{0,1\}^{c \cdot d} \rightarrow \{0,1\}^c$ be the function that this algorithm computes. That is, $\hat{f}(\hbx)=\ca(\hbx)$.
Assume that the function $\hat{f}$ can be computed by a threshold circuit $T$ of size $\co(d\log^2(d))$. We will construct a neural network $N$ of size $\co(d\log^2(d))$ such that $\snorm{f-N}_{L_2(\mu)} \leq \epsilon$ and thus reach a contradiction. It implies that $\hat{f}$ can be computed in polynomial time but cannot be computed by  threshold circuits of size 
$\co(d\log^2(d))$.
%$cd\log(d)$.
Then, the theorem follows from the following lemma.
% (see proof in Section~\ref{sec:proof of lemma to output dim 1 size m}).
\begin{lemma}
	\label{lemma:to output dim 1 size m}
	Let $l(d) = \co(\log(d))$ be monotonically non-decreasing and let $g:\{0,1\}^{d \cdot l} \rightarrow \{0,1\}^l$ be a function that can be computed in polynomial time, and cannot be computed by threshold circuits of size $\co(d\log^2(d))$. Then, there is a function $g':\{0,1\}^{d'} \rightarrow \{0,1\}$ in $\Ptime$, that cannot be computed by threshold circuits of size $\co(d')$.
\end{lemma}
\begin{proof}
	We first define $g'$ in a similar manner to the proof of Lemma~\ref{lemma:to output dim 1}.
	Let $g':\{0,1\}^{d'} \rightarrow \{0,1\}$ be a function such that if $d' = d \cdot l + l$ then we have the following. Let $\bx \in \{0,1\}^{d'}$ and denote $\bx^1 = (x_1,\ldots,x_{d \cdot l})$ and $\bx^2 = (x_{d \cdot l+1},\ldots x_{d \cdot l+l})$. If $\bx^2$ has a $1$-bit in the $i$-th coordinate and all other bits are $0$, then we say that $\bx^2$ is the {\em $i$-selector}. For $\bx \in \{0,1\}^{d'}$ such that $\bx^2$ is $i$-selector, we have $g'(\bx)=(g(\bx^1))_i$. Namely, $g'$ returns the $i$-th output bit of $g(\bx^1)$.
	
	Since $g$ can be computed in polynomial time then clearly $g'$ can also be computed in polynomial time. 
	Assume that $g'$ can be computed by a threshold circuit $T'$ of size $c' \cdot d'$ for some constant $c'$. Then, $g$ can be computed by a threshold circuit $T$ of size $\co(d\log^2(d))$ as follows. The circuit $T$ consists of $l$ circuits $T_1,\ldots,T_l$, such that $T_i$ computes the $i$-th output bit. The circuit $T_i$ has input dimension $d \cdot l$, and is obtained from $T'$ by hardwiring the input bits $\bx^2$ to be the $i$-selector. That is, let $n$ be a threshold gate in the first layer of $T'$, and assume that the weight from the $i$-th component of $\bx^2$ to $n$ is $w$, and that the bias of $n$ is $b$. Then, in $T_i$ we change the bias of $n$ to $b+w$. Note that $T$ has size $l \cdot c'd' = c'l(d l + l)=\co(d\log^2(d))$, and that $T$ computes $g$.
\end{proof}

The construction of the network $N$ is done in a similar manner to the proof of Theorem~\ref{thm:bounded depth}, with some necessary modifications.
Note that here the size of $N$ should be $\co(d\log^2(d))$. However, the transformation from $\bx \in [0,1]^d$ to $\bin(\trunc(\bx))$ from Lemma~\ref{lemma:to binary} requires a larger size. Hence, we use here the construction from the proof of Lemma~\ref{lemma:to binary exponential}. Then, transforming $\bx \in [0,1]^d$ to $\bin(\trunc(\bx)) \in \{0,1\}^{c \cdot d}$ requires only $\co(d \cdot c^2)$ neurons. Indeed, for every $i \in [d]$, the computation of each bit in $\bin(\trunc(x_i))$ requires at most $\co(c)$ neurons ($\co(c)$ layers of constant width). Since $c=\co(\log(d))$, then the total size required for this transformation is $\co(d \log^2(d))$.
Now, note that the simulation of the threshold circuit $T$ can be done by Lemma~\ref{lemma:from TC to NN}. Since $T$ is of size $\co(d \log^2(d))$ then simulating $T$ with a neural network requires size $\co(d \log^2(d))$.
Overall, the size of $N$ is $\co(d \log^2(d))$.

\section{Proofs for Section~\ref{sec:boolean}}

\subsection{Proof of Theorem~\ref{thm:CC_ReLU}}
%\section{Proof of Theorem~\ref{thm:CC_ReLU}}
\label{app:proof of theorem cc ReLU}

For a Boolean function $f$, we denote by $R_{\epsilon}(f)$ the randomized communication complexity with error $\epsilon$. Namely, the minimal cost of a protocol that computes $f$ correctly with probability at least $1-\epsilon$ on every input and partition. Note that  $R(f) = R_{1/3}(f)$.
The following lemma, which provides an upper bound on the communication complexity of LTFs, was shown by \cite{nisan1993communication} :
\begin{lemma}
	\label{lemma:from nisan}
	Let $g=L_{\ba=(a_1 \ldots a_m), \theta}$ be an LTF. Then $R_{\epsilon}(g)=\co(\log m + \log \epsilon^{-1})$.
\end{lemma}
We stress that no assumption in this Lemma is made on the real weights of the LTF $g$. The proof builds on a classical result 
that every LTF with $m$ inputs can be computed by an LTF with integer weights 
%of total weight 
whose absolute values are at most 
$\co(2^{m \log m})$ (cf. \cite{goldmann1998simulating,goldmann1992majority}).
	
We first handle the case of ReLU networks. Then, we will explain how to extend the proof to the case of $k$-piecewise-linear activation functions.
\begin{lemma}
\label{lemma:handle ReLU}
	Let $N$ be a ReLU network, computing a function $h:\{0,1\}^d \rightarrow \{0,1\}$ with $R(h)=\Omega(d)$. Then, $N$ has size $\Omega(d/\log d)$. 
\end{lemma}
\begin{proof}
	We first outline the proof for depth-$2$ networks.
	Suppose that there is a depth-$2$ ReLU network with $s$ hidden neurons computing the function $h$ . 
	Assume w.l.o.g. that $s=\co(d)$.
	Using Lemma~\ref{lemma:from nisan}, the two parties can determine whether the output of each of the $s$ hidden neurons is zero or positive, where the probability of error for each neuron is
	$\co(d^{-2})$ and the total amount of communication is $\co(s \log d)$. With this information each active ReLU neuron becomes a linear function and the whole network ``collapses" to a single linear function, for which we wish to determine the sign of the output. Thus, it remain to compute a single threshold gate. The parties can infer the ``updated" weights of this threshold gate, namely the real coefficient of every input $x_i$ in the threshold gate (observe that we may and do assume that both players know the weights of all neurons). Then, by using Lemma~\ref{lemma:from nisan} again, the parties can determine the sign of the output, with an additional communication cost of $\co(\log d)$ bits. Overall, the cost of the protocol is $\co(s \log d)$, and it succeeds with probability $1-o(1)$. Since by our assumption the cost of the protocol must be $\Omega(d)$, then $s=\Omega(d/\log d)$.
	
	As an illustrative example consider a ReLU network with $6$ inputs $x_1, \ldots x_6$ and three ReLUs at the hidden layer $[ax_1+a'x_2]_+, [bx_3+b'x_4]_+$ and $[cx_5+c'x_6]_+$ (where $a,a',b,b',c,c'$ are real constants) feeding to an output neuron that computes the sum (i.e., all weights are $1$). If we know that the first gate evaluates to $0$ and the others are positive, then the output of the network is positive iff $bx_3+b'x_4+cx_5+c'x_6 > 0$.

	For depth greater than $2$ we can use the same method to deduce whether each ReLU neuron has positive or zero output starting with the neurons in the first hidden layer and then doing the same evaluation (positive vs. zero) for all neurons, evaluating all neurons of depth $i$ before neurons of depth $i+1$ (and evaluating the neurons in a given layer in an arbitrary order). Once we know for all hidden neurons whether their output is positive or zero, the players can infer the linear function feeding into the topmost output neuron and evaluate its sign. 
\end{proof}

Neural networks with the ReLU activation function are a special case of more general networks where the activation function $\sigma:\reals \rightarrow \reals$ of each neuron is \emph{piecewise linear}. Namely, there are $k$ real numbers $c_1<c_2<...<c_k$ such that if $c_i < x \leq c_{i+1}$ (for $1 \leq i < k$) then $\sigma(x)=a_ix+b_i$, if $x \leq c_1$ then $\sigma(x)=a_0x+b_0$, and else $\sigma(x)=a_kx+b_k$, where for every $i$ the parameters $a_i,b_i$ are real numbers\footnote{No further assumptions are made on $\sigma$: it does not have to be continuous nor monotone.}. 
We now explain how to extend the communication complexity argument from Lemma~\ref{lemma:handle ReLU} to prove lower bounds for such networks. 
Suppose 
that the network $N$ that computes $h$ has size $s=\co(d)$ and $k$-piecewise-linear activation functions. Using binary search on $[k]$ and Lemma~\ref{lemma:from nisan}, we can determine for each neuron (starting with the neurons in the first hidden layer and moving upward), which of the $k$ linear functions is used in the output of the neuron, using $\co(\log d \log k)$ bits communicated per neuron.
Once we determine this information for all neurons we have the linear function of the output neuron and we can evaluate the sign of the output with additional $\co(\log d)$ bits. The overall communication (also ensuring that the probability of error at every neuron is at most $\co(d^{-2})$) is $\co(s \log d \log k)$ and by the union bound\footnote{We can assume $k=2^{\co(d/\log d)}$ as otherwise the claim in the theorem is trivial.}, the probability of error is $o(1)$. 
Since by our assumption the cost of the protocol must be $\Omega(d)$, then $s=\Omega(d / (\log d \log k))$.

\subsection{Proof of Theorem~\ref{thm:IP approx}}
%\section{Proof of Theorem~\ref{thm:IP approx}}
\label{app:proof of theorem IP approx}

Let $f=\text{IP}_d$, and let $N$ be a ReLU network of size $s$ such that $\snorm{N-f}_{L_2(\cu(\{0,1\}^{2d}))} \leq \epsilon$. Let $N'$ be a neural network of size $s+2$ such that for every $\bz \in \{0,1\}^{2d}$ we have: if $N(\bz) \leq \frac{1}{3}$ then $N'(\bz)=0$, and if $N(\bz) \geq \frac{2}{3}$ then $N'(\bz)=1$. Such a network can be obtained from $N$ by adding two neurons, namely, 
\[
    N'(\bz) = \left[3N(\bz)-1\right]_+ - \left[3N(\bz)-2\right]_+~.
\]
Note that for every $\bz \in \{0,1\}^{2d}$ such that $|N(\bz)-f(\bz)| \leq \frac{1}{3}$, we have $N'(\bz)=f(\bz)$. 
Also, we have 
\begin{align*}
    \epsilon^2 
    &\geq \snorm{N-f}_{L_2(\cu(\{0,1\}^{2d}))}^2 
    = \E_{\bz \sim \cu(\{0,1\}^{2d})}\left(N(\bz)-f(\bz)\right)^2
    \\
    &\geq \left(\frac{1}{3}\right)^2 \cdot \Pr_{\bz \sim \cu(\{0,1\}^{2d})}\left[|N(\bz)-f(\bz)| > \frac{1}{3} \right]~,
\end{align*}
and therefore $\Pr_{\bz \sim \cu(\{0,1\}^{2d})}\left[|N(\bz)-f(\bz)| > \frac{1}{3} \right] \leq 9 \epsilon^2$.
Thus, with probability at least $1-9 \epsilon^2$ over $\bz \sim \cu(\{0,1\}^{2d})$ we have $N'(\bz)=f(\bz)$.

We now use $N'$ to obtain a protocol that computes $f$ w.h.p. for every input $\bz \in \{0,1\}^{2d}$ and every partition. Let $\bx,\by \in \{0,1\}^d$. In order to compute $\text{IP}_d(\bx,\by)$, the players first use their shared randomness to generate $\bx',\by' \sim \cu(\{0,1\}^{d})$. Note that
\begin{equation}
\label{eq:IP}
	\text{IP}_d(\bx,\by) 
	= \text{IP}_d(\bx+\bx',\by+\by') + \text{IP}_d(\bx+\bx',\by') + \text{IP}_d(\bx',\by+\by') + \text{IP}_d(\bx',\by') \mod 2~.
\end{equation}
Also, note that $\bx+\bx'$ and $\by+\by'$ are distributed uniformly on $\{0,1\}^d$ (where the addition is$\mod 2$).
Thus, by computing $\text{IP}_d(\bx+\bx',\by+\by'), \text{IP}_d(\bx+\bx',\by'), \text{IP}_d(\bx',\by+\by'), \text{IP}_d(\bx',\by')$ the players can compute $\text{IP}_d(\bx,\by)$.
Note that by the union bound, with probability at least $1 - 4 \cdot 9\epsilon^2$ over the choice of $\bx',\by'$ we have $N'(\bx+\bx',\by+\by')=\text{IP}_d(\bx+\bx',\by+\by')$, $N'(\bx+\bx',\by')=\text{IP}_d(\bx+\bx',\by')$, $N'(\bx',\by+\by')=\text{IP}_d(\bx',\by+\by')$ and $N'(\bx',\by')=\text{IP}_d(\bx',\by')$.

Since both players know $\bx',\by'$, then they can compute $N'(\bx',\by')$ without communicating. Now, the players compute the signs of $N'(\bx+\bx',\by+\by'),N'(\bx+\bx',\by'),N'(\bx',\by+\by')$ using the protocol described in the proof of Lemma~\ref{lemma:handle ReLU} (we assume here w.l.o.g. that $s=\co(d)$). It will succeed with probability $1-o(1)$, and its cost is $\co(s \log d)$.
Finally, the players compute $\text{IP}_d(\bx,\by)$ with Eq.~\ref{eq:IP}. Overall, the probability for an error is at most $4 \cdot 9\epsilon^2 + o(1)$. 
By plugging in $\epsilon=\frac{1}{20}$, we obtain that this probability is at most $1/3$.
Since there is a linear lower bound on the randomized communication complexity of $\text{IP}_d$, and since we obtained a protocol with cost $\co(s \log d)$, then we have $s=\Omega(d/\log d)$.

%\section{Linear lower bounds for Disjointness and Inner-Product}
\subsection{Proof of Theorem~\ref{thm:CC_real}}
%\section{Proof of Theorem~\ref{thm:CC_real}}
\label{app:proof of theorem CC real}

\begin{lemma}
Let $g=L_{\ba=(a_1 \ldots a_d), \theta}$ be an LTF. Then $CC^{\mathbb{R}}(g)=1$.
\end{lemma}
\begin{proof}
As before we may assume that the players know the weights and the bias of the LTF.
Suppose that $A \subseteq [d]$ is the set of indices of bits Alice gets and $B \subseteq [d]$ is the set of indices of bits Bob gets. Alice sends $\alpha=\sum_{i \in A} a_ix_i$ and Bob sends $\beta=\theta-\sum_{j \in B} a_jy_j$.
The output of $f$ can be decided based on whether $\alpha > \beta$,
% (in which case the referee can write a $1$ to the word or $0$ otherwise), 
concluding the proof.
\end{proof}

Since we can use real communication to evaluate the sign of the output of a ReLU neuron and of a threshold gate with one round of communication, we get using a similar argument to the proof of Lemma~\ref{lemma:handle ReLU}: %Theorem~\ref{thm:CC_ReLU}:
\begin{corollary}
Let $f$ be a Boolean function with $d$ inputs. 
Suppose that $CC^{\mathbb{R}}(f)= \Omega(d)$. Then, any threshold circuit or ReLU network computing $f$ has size $\Omega(d)$.
\end{corollary}

Moreover, for a $k$-piecewise-linear activation function, we can determine which linear function is active in the output of a neuron with $\co(\log k)$ cost in the real communication model, using binary search. A similar reasoning to the proof of  Theorem~\ref{thm:CC_ReLU} yields:
\begin{corollary}
Let $f$ be a Boolean function with $d$ inputs. 
Suppose that $CC^{\mathbb{R}}(f)= \Omega(d)$. Then, any neural network with a $k$-piecewise-linear activation function computing $f$ has size $\Omega(d/\log k)$.
\end{corollary}

\subsection{Real communication complexity of $\text{IP}_d$}
%\section{Real communication complexity of $\text{IP}_d$}
\label{app:IP real bound}

\begin{theorem}
We have $CC^{\mathbb{R}}(\text{IP}_d) = \Omega(d)$.
\end{theorem}
\begin{proof}
Recall that the \emph{communication matrix} of a Boolean function $f(\bx,\by)$ where $\bx$ and $\by$ are binary vectors of length $d$
is a $2^d \times 2^d$ matrix $M_f$ where $M_f(\bx,\by)=f(\bx,\by)$. By \cite{chattopadhyay2019equality} (Lemma 3.5 and Lemma 3.7) if $M_f$ has $\alpha 2^{2d}$ ones and any
$1$-monochromatic rectangle $R$ has $|R| \leq \beta 2^{2d}$, it holds that $CC^{\mathbb{R}}(f)= \Omega(\log(\alpha (\beta^{\eta-1}))$ for any $\eta \in (1/2,1)$.
By Lindsey's Lemma, any $1$-monochromatic rectangle $R$ satisfies $|R| \leq 2^d$. It follows that for $f=\text{IP}_d$ we have $\beta \leq 2^{-d}$. As $\alpha$ is roughly $1/2$
we have that $CC^{\mathbb{R}}(\text{IP}_d)=\Omega(d)$.
\end{proof}

\subsection{Proof of Theorem~\ref{thm:boolean upper bound}}
%\section{Proof of Theorem~\ref{thm:boolean upper bound}}
\label{app:proof of theorem boolean upper bound}

Given input $\bx,\by \in \{0,1\}^d$, we have:
\[
    \text{DISJ}_d(\bx,\by) = \neg \bigvee_{i \in [d]} (x_i \wedge y_i)~.
\]
Since $x_i \wedge y_i = [x_i+y_i-1]_+$ then implementing $x_i \wedge y_i$ requires a single neuron for every $i \in [d]$. Implementing $\neg \bigvee_i z_i = [\sum_i(-z_i)+1]_+$ also requires a single neuron.

Likewise, we have
\[
    \text{IP}_d(\bx,\by) = \bigoplus_{i \in [d]} (x_i \wedge y_i)~.
\]
Implementing $x_i \wedge y_i$ requires a single neuron for every $i \in [d]$. Implementing $\bigoplus_{i} z_i$ requires evaluating the parity of $\sum_i z_i$. Computing the parity bit of an integer $1 \leq j \leq d$ can be done by a network of size $\co(d)$ using a straightforward construction, similar to the construction from the proof of Lemma~\ref{lemma:to binary}.

\section{Proofs for Section~\ref{sec:separation real}}

\subsection{Proof of Proposition~\ref{prop:from bool to real}}
%\section{Proof of Proposition~\ref{prop:from bool to real}}
\label{app:proof of prop from bool to real}

For $\bz \in \{0,1\}^d$ we denote $A_\bz = \{\bx \in [0,1]^d: \forall i \in [d], \; |x_i-z_i| \leq \frac{1}{4}\}$. Thus, $A_\bz$ is a cube with volume $\left(\frac{1}{4}\right)^d$. For $\bx \in [0,1]^d$ we denote $\text{dist}(\bx,A_\bz)=\min\{\norm{\bx-\ba}: \ba \in A_\bz\}$. For $\bz \in \{0,1\}^d$, let $h_\bz:[0,1]^d \rightarrow [0,1]$ be such that $h_\bz(\bx) = \max\{0, 1 - 4 \cdot \text{dist}(\bx,A_\bz)\}$. Note that if $\bx \in A_\bz$ then $h_\bz(\bx)=1$, if there is $i \in [d]$ such that $|x_i-z_i| \geq \frac{1}{2}$ then $h_\bz(\bx)=0$, and $h_\bz$ is $4$-Lipschitz. Let $f:[0,1]^d \rightarrow [0,1]$ be such that $f(\bx) = \sum_{\bz \in \{0,1\}^d}g(\bz) h_\bz(\bx)$.
Note that for every $\bz \in \{0,1\}^d$ and $\bx \in A_\bz$ we have $f(\bx) = g(\bz)$. Moreover, since for every $\bx \in [0,1]$ there is at most one $\bz \in \{0,1\}^d$ such that $h_\bz(\bx) \neq 0$, then $f$ is also $4$-Lipschitz. Let $\mu$ be the uniform distribution on $([0,1/4] \cup [3/4,1])^d = \bigcup_{\bz \in \{0,1\}^d}A_\bz$. Note that $\mu$ has a polynomially-bounded marginal density.

\paragraph{Part (1).}
Let $g':\{0,3/4\}^d \rightarrow \{0,1\}$ be a function that corresponds to $g$, namely, $g'(\bz) = g(\frac{4}{3} \bz)$ for every $\bz \in \{0,3/4\}^d$. Since $g$ cannot be $\epsilon$-approximated by networks of size $\co(m)$ w.r.t. $\cu(\{0,1\}^d)$, then $g'$ cannot be $\epsilon$-approximated by networks of size $\co(m)$ w.r.t. $\cu(\{0,3/4\}^d)$. 
Assume that there is a neural network $N$ of size $\co(m)$ such that 
%$\E_{\bx \sim \mu}\left(N(\bx)-f(\bx)\right)^2 \leq \epsilon$. 
$\snorm{N-f}_{L_2(\mu)} \leq \epsilon$.
We show that there exists a network $N'$ of the same size such that 
%$\E_{\bx \sim \cu(\{0,3/4\}^d)}\left(N'(\bx)-g'(\bx)\right)^2 \leq \epsilon$. 
$\snorm{N'-g'}_{L_2(\cu(\{0,3/4\}^d))} \leq \epsilon$.
Thus, if $f$ can be $\epsilon$-approximated by a network of size $\co(m)$ then $g'$ can also be $\epsilon$-approximated by a network of size $\co(m)$, and hence we reach a contradiction.	

For every $\bc \in [0,1/4]^d$ we denote by $N_\bc$ the neural network of size $\co(m)$ such that for every $\bx$ we have $N_\bc(\bx)=N(\bx+\bc)$. The network $N_\bc$ is obtained from $N$ by adding the appropriate bias terms to the neurons in the first hidden layer, and hence has size $\co(m)$. We now show that there exists $\bc \in [0,1/4]^d$ such that 
%$\E_{\bz \sim \cu(\{0,3/4\}^d)} \left(N_\bc(\bz)-g'(\bz)\right)^2 \leq \epsilon^2$.
$\snorm{N_\bc - g'}_{L_2(\cu(\{0,3/4\}^d))} \leq \epsilon$.

We have 
\begin{align*}
	\E_{\bc \sim \cu([0,1/4]^d)} \E_{\bz \sim \cu(\{0,3/4\}^d)} \left(N_\bc(\bz)-g'(\bz)\right)^2 
	&= \E_{\bc \sim \cu([0,1/4]^d)} \E_{\bz \sim \cu(\{0,3/4\}^d)} \left(N(\bz+\bc)-f(\bz+\bc)\right)^2 
	\\
	&= \E_{\bx \sim \mu} \left(N(\bx)-f(\bx)\right)^2 
	\\
	&=\snorm{N-f}_{L_2(\mu)}^2
	\leq \epsilon^2~.
\end{align*}
Hence, there exists $\bc$ such that 
$\E_{\bz \sim \cu(\{0,3/4\}^d)} \left(N_\bc(\bz)-g'(\bz)\right)^2 \leq \epsilon^2$,
and thus $\snorm{N_\bc - g'}_{L_2(\cu(\{0,3/4\}^d))} \leq \epsilon$.

\paragraph{Part (2).}
Let $N$ be a neural network of size $\tilde{m}$ such that for every $\bz \in \{0,1\}^d$ we have $N(\bz)=g(\bz)$. Let $\tilde{N}$ be a network, that first transforms the input $\bx \in ([0,1/4] \cup [3/4,1])^d$ to $\bz \in \{0,1\}^d$ by rounding each component, and then computes $N(\bz)$.
Note transforming $x_i \in [0,1/4] \cup [3/4,1]$ to the corresponding $z_i \in \{0,1\}$ can be done with two neurons as follows:
\[
z_i	= \left[2x_i-\frac{1}{2}\right]_+ - \left[2x_i-\frac{3}{2}\right]_+~.
\]
Hence, the size of $\tilde{N}$ is $\tilde{m}+2d$. Also, we have
\[
\snorm{\tilde{N}-f}_{L_2(\mu)} ^2
= \E_{\bx \sim \cu(([0,1/4] \cup [3/4,1])^d)}\left(\tilde{N}(\bx)-f(\bx)\right)^2
= \E_{\bz \sim \cu(\{0,1\}^d)}\left(N(\bz) - g(\bz) \right)^2
= 0~.
\]

\paragraph{Part (3).}
If $g \in \Ptime$, then $f$ can be computed in polynomial time as follows. Given an input $\bx$ we find the nearest $\bz \in \{0,1\}^d$, compute $g(\bz)$, and return $f(\bx)= g(\bz) \cdot \max\{0, 1 - 4 \cdot \text{dist}(\bx,A_\bz)\}$.

\subsection{Proof of Theorem~\ref{thm:separation benign infty}}
%\section{Proof of Theorem~\ref{thm:separation benign infty}}
\label{app:proof of theorem separation benign infty}

Let $g:\{0,1\}^d \rightarrow \{0,1\}$ be either the disjointness function or the inner product function.
Let $f:[0,1]^d \rightarrow [0,1]$ be the function computed by the neural network $\cn$ from Theorem~\ref{thm:boolean upper bound} that corresponds to $g$. Thus, for every $\bx \in \{0,1\}^d$ we have $f(\bx)=g(\bx)$.

Given an input $\bx \in [0,1]^d$, we can construct in time polynomial in $d$ the network $\cn$, as we describe in the proof of Theorem~\ref{thm:boolean upper bound}, and hence we can compute $f(\bx)$ in polynomial time. Moreover, by the construction in Theorem~\ref{thm:boolean upper bound}, the network $\cn$ is of a constant depth and size $\co(d)$, and the absolute values of its weights are bounded by $\poly(d)$, and therefore it computes a $\poly(d)$-Lipschitz function. Hence, $f$ is polynomial-time benign, and it is computed by a network of size $\co(d)$. 

Let $N$ be a neural network such that $\norm{f-N}_\infty \leq \frac{1}{3}$. Let $N'$ be a network such that if $N(\bx) \leq \frac{1}{3}$ then $N'(\bx)=0$, and if $N(\bx) \geq \frac{2}{3}$ then $N'(\bx)=1$. Such a network can be easily obtained from $N$. Indeed, let $y=N(\bx)$, then we define 
\[
    N'(\bx) = \left[3y-1\right]_+ - \left[3y-2\right]_+~.
\]
Since for every $\bx \in [0,1]^d$ we have $|f(\bx)-N(\bx)| \leq \frac{1}{3}$, then for every $\bx$ such that $f(\bx) \in \{0,1\}$ we have $N'(\bx)=f(\bx)$. The function $f$ is such that for every $\bx \in \{0,1\}^d$ we have $f(\bx) = g(\bx) \in \{0,1\}$, and thus $N'(\bx)=f(\bx)=g(\bx)$. Hence, $N'$ computes $g$. By Corollary~\ref{cor:CC_real}, the size of $N'$ is $\Omega(d)$, and therefore the size of $N$ is also $\Omega(d)$.

\end{document}